\documentclass{article}
\usepackage[margin=1.5in]{geometry}
\usepackage{subfig}

\usepackage[utf8]{inputenc} % allow utf-8 input
\usepackage[T1]{fontenc}    % use 8-bit T1 fonts
\usepackage{hyperref}       % hyperlinks
\usepackage{url}            % simple URL typesetting
\usepackage{booktabs}       % professional-quality tables
\usepackage{amsfonts}       % blackboard math symbols
\usepackage{nicefrac}       % compact symbols for 1/2, etc.
\usepackage{microtype}      % microtypography

\usepackage{amsthm}
\usepackage{amsmath}
\usepackage{amssymb}
\usepackage{color}
\usepackage{mathtools}
\usepackage{enumitem}
\usepackage{newfloat}
\usepackage{caption}
 \usepackage{quoting}
\DeclareCaptionType[
    name=,
    placement=tbhp,
    within=none
]{algorithmfloat}

\renewcommand{\figurename}{Algorithm}

\newtheorem{definition}{Definition}
\newtheorem{theorem}{Theorem}
\newtheorem{lemma}{Lemma}

\newtheorem{proposition}{Proposition}

\newcommand{\mb}{\mathbb}
\newcommand{\mc}{\mathcal}

\newcommand{\prob}{\mb{P}}
\newcommand{\event}{\mc{E}}
\newcommand{\argmax}{\text{argmax}}
\newcommand{\argmin}{\text{argmin}}
\newcommand{\E}{\mb{E}}
\renewcommand{\P}{\mb{P}}

\newcommand{\hatparam}{\widehat{\theta}}

\newcommand{\thetahat}{\widehat{\theta}}
\newcommand{\thetastar}{\theta^{\ast}}
\newcommand{\xset}{\mc{X}}
\newcommand{\yset}{\mc{Y}}

\newcommand{\sset}{\mc{S}}

\newcommand{\zhatset}{\widehat{\mc{Z}}}
\newcommand{\conv}{\text{conv}}
\newcommand{\arm}{x}

\newcommand{\yarm}{y}

\newcommand{\reward}{r}

\newcommand{\armseq}{\mathbf{x}}
\newcommand{\design}{\lambda}
\newcommand{\simplex}{\Sigma}

\newcommand{\support}{p}

\newcommand{\armindex}{i}
\newcommand{\sampindex}{j}
\newcommand{\phaseindex}{t}

\newcommand{\randnumsamples}{N}

\def\calX{\mathcal{X}}
\def\calzhat{\widehat{\mathcal{Z}}}
\def\calZ{\mathcal{Z}}
\def\calC{\mathcal{C}}
\def\calY{\mathcal{Y}}
\def\calS{\mathcal{S}}
\def\calN{\mathcal{N}}
\def\calE{\mathcal{E}}
\def\calV{\mathcal{V}}
\newcounter{alg}

\renewcommand{\simplex}{
  \mathchoice
    {\includegraphics[height=1.4ex, trim= 0pt 30pt 0pt 0pt]{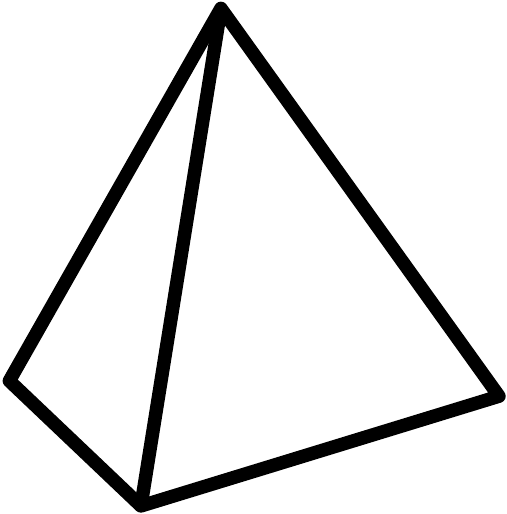}} % \displaystyle
    {\includegraphics[height=1.4ex, trim= 0pt 25pt 0pt 0pt]{simplex.pdf}} % \textstyle
    {\includegraphics[height=1.1ex, trim=0pt 30pt 0pt 0pt]{simplex.pdf}} % \scriptstyle
    {\includegraphics[height=.8ex, trim=0pt 30pt 0pt 0pt]{simplex.pdf}} % \scriptscriptstyle
}

\title{Sequential Experimental Design for Transductive Linear Bandits}

\author{%
  Tanner Fiez
  \thanks{Department of Electrical and Computer Engineering, University of Washington, fiezt@uw.edu}
  \and
  Lalit Jain\thanks{Paul G. Allen School of Computer Science \& Engineering, University of Washington, lalitj@cs.washington.edu, contribution shared equally among first two authors}
  \and
  Kevin Jamieson\thanks{Paul G. Allen School of Computer Science \& Engineering, University of Washington, jamieson@cs.washington.edu}
  \and
  Lillian Ratliff\thanks{Department of Electrical and Computer Engineering, University of Washington, ratliffl@uw.edu }
}

\begin{document}

\maketitle
\begin{abstract}
In this paper we introduce the \emph{transductive linear bandit problem}: given a set of measurement vectors $\mathcal{X}\subset \mathbb{R}^d$, a set of items $\mathcal{Z}\subset \mathbb{R}^d$, a fixed confidence $\delta$, and an unknown vector $\theta^{\ast}\in \mathbb{R}^d$, the goal is to infer $\argmax_{z\in \mathcal{Z}} z^\top\theta^\ast$ with probability $1-\delta$ by making as few sequentially chosen noisy measurements of the form $x^\top\theta^{\ast}$ as possible.
When $\mathcal{X}=\mathcal{Z}$, this setting generalizes \emph{linear bandits}, and when $\mathcal{X}$ is the standard basis vectors and $\mathcal{Z}\subset \{0,1\}^d$, \emph{combinatorial bandits}.
Such a transductive setting naturally arises when the set of measurement vectors is limited due to factors such as availability or cost.
As an example, in drug discovery the compounds and dosages $\mathcal{X}$ a practitioner may be willing to evaluate in the lab in vitro due to cost or safety reasons may differ vastly from those compounds and dosages $\mathcal{Z}$ that can be safely administered to patients in vivo.
Alternatively, in recommender systems for books, the set of books $\calX$ a user is queried about may be restricted to well known best-sellers even though the goal might be to recommend more esoteric titles $\mathcal{Z}$.
In this paper, we provide instance-dependent lower bounds for the transductive setting, an algorithm that matches these up to logarithmic factors, and an evaluation. In particular, we provide the first non-asymptotic algorithm for linear bandits that nearly achieves the information theoretic lower bound.
\end{abstract}

% \subsection{Pyramid test}
% \begin{equation}
%   \sum_{\simplex} \sum_{i \in \simplex} s_i  \quad
%   x \simplex^{x \simplex^{x \simplex}}
% \end{equation}
% Define
% \begin{equation}
% \simplex_{\calX} := \Big\{\lambda\in \mathbb{R}^{|\calX|}:\lambda\geq 0, \sum_{x\in\calX} \lambda_x = 1\Big\}
% \end{equation}

\section{Introduction}

In content recommendation or property optimization in the physical sciences, often there is a set of items (e.g., products to purchase, drugs) described by a set of feature vectors $\calZ\subset \mb{R}^d$, and the goal is to find the $z \in \mc{Z}$ that maximizes some response or property (e.g., affinity of user to the product, drug combating disease).
A natural model for these settings is to assume that there is an unknown vector $\theta^{\ast} \in \mb{R}^d$ and the expected response to any item $z \in \mc{Z}$, if evaluated, is equal to $z^\top \theta^{\ast}$.
However, we often cannot measure $z^\top \theta^{\ast}$ directly, but we may infer it transductively through some potentially noisy probes.
That is, given a finite set of probes $\calX\subset \mathbb{R}^d$ we observe $x^\top\theta^{\ast} + \eta$ for any $x \in \calX$ where $\eta$ is independent mean-zero, sub-Gaussian noise.
Given a set of measurements $\{(x_i,r_i)\}_{i=1}^N$ one can construct the least squares estimator $\widehat{\theta} = \arg\min_\theta \sum_{i=1}^N (r_i - x_i^\top \theta)^2$ and then use $\widehat{\theta}$ as a plug-in estimate for $\theta^{\ast}$ to estimate the optimal $z_* := \argmax_{z\in\calZ} z^{\top}\theta^{\ast}$.
However, the accuracy of such a plug-in estimator depends critically on the number and choice of probes used to construct $\widehat{\theta}$.
Unfortunately, the optimal allocation of probes cannot be decided a priori: it must be chosen sequentially and adapt to the observations in real-time to optimize the accuracy of the prediction.

If the probing vectors $\mc{X}$ are \emph{equal} to the item vectors $\mc{Z}$, this problem is known as \emph{pure exploration for linear bandits} which is considered in \cite{karnin2016verification,soare2014best, tao2018best,xu2018fully}.
This naturally arises in content recommendation, for example, if $\mc{X}=\mc{Z}$ is a feature representation of songs, and $\theta^{\ast}$ represents a user's music preferences, a music recommendation system can elicit the preference for a particular song $z \in \mc{Z}$ directly by enqueuing it in the user's playlist.
However, often times there are constraints on which items in $\mc{Z}$ can be shown to the user.
\begin{enumerate}[topsep=0pt, partopsep=0pt, labelindent=0pt, labelwidth=0pt, leftmargin=12pt]
    \setlength{\itemsep}{0pt}%
    \setlength{\parskip}{0pt}%
  \item $\mc{X} \subset \mc{Z}$. Consider a whiskey bar with hundreds of whiskies ranging in price from dollars a shot to hundreds of dollars.
The bar tender may have an implicit feature representation of each whiskey, the patron has an implicit preference vector $\theta^{\ast}$, and the bar tender wants to select the affordable whiskeys $\mc{X} \subset \mc{Z}$ in a taste test to get an idea of the patron's preferences before recommending the expensive whiskies that optimize the patron's preferences in $\mc{Z}$.
  \item $\mc{Z} \subset \mc{X}$. In drug discovery, thousands of compounds are evaluated in order to determine which ones are effective at combating a disease.
However, it may be that while $\mc{Z}$ is the set of compounds and doses that are approved for medical use (e.g., safe), it may be advantageous to test even unsafe compounds or dosages $\mc{X}$ such that $\mc{X} \supset \mc{Z}$. Such unsafe $\mc{X}$ may aid in predicting the optimal $z_\ast \in \mc{Z}$ because they provide more information about $\theta^\ast$.
  \item $\mc{Z} \cap \mc{X} = \emptyset$. Consider a user shopping for a home among a set $\mc{Z}$ where each is parameterized by a number of factors like distance to work, school quality, crime rate, etc. so that each $z \in \mc{Z}$ can be described as a linear combination of the relevant factors described by $\mc{X}$: $z = \sum_{x \in \mc{X}} \alpha_{z,x} x$, where we may take each $x \in \mc{X}$ to simply be one-hot-encoded.
The response $x^\top \theta^{\ast} + \eta$ reflects the user's preferences for the query $x$, a specific attribute of the house.
Indeed, if all $\alpha_{z,x} \in \{0,1\}$ this is known as \emph{pure exploration for combinatorial bandits} \cite{chen2017nearly, cao2017disagreement}. That is, a house either has the attribute, or not.
\end{enumerate}

%{\color{red} Transductive bandits naturally arise in many settings already considered in the multi-armed bandits literature. For example when $\calX =\calZ$, the problem reduces to \emph{best-arm identification for linear bandits}  and when $\calZ\subset \{0,1\}^n$ and $\calX$ is the standard basis vectors, the problem reduces to pure exploration for combinatorial bandits .
%The main difference between the standard multi-armed bandit problem and the more general linear bandit setting is that pulls of one arm can give us information about another.}

Given items $\mc{Z}$, measurement probes $\mc{X}$, a confidence $\delta$, and an unknown $\theta^*$, this paper develops algorithms to sequentially decide which measurements in $\mc{X}$ to take in order to minimize the number of measurements necessary in order to determine $z_*$ with high probability.

\subsection{Contributions}
Our goals are broadly to first define the transductive bandit problem and then characterize the instance-optimal sample complexity for this problem.  Our contributions include the following.
\begin{enumerate}[topsep=0pt, partopsep=0pt, labelindent=0pt, labelwidth=0pt, leftmargin=12pt]
    \setlength{\itemsep}{0pt}%
    \setlength{\parskip}{0pt}%
    \item In Section~\ref{sec:LinearExperimentalDesign} we provide instance dependent lower bounds for the transductive bandit problem that simultaneously generalize previous known lower bounds for linear bandits and combinatorial bandits using standard arguments.
    \item In Section~\ref{sec:alg} the main contribution of this paper, we give an algorithm (Algorithm~\ref{alg:main}) for transductive linear bandits and prove an associated sample complexity result (Theorem~\ref{thm:SampleComplexity}). We show that the sample complexity we obtain matches the lower bound up to logarithmic factors. Along the way, we discuss how rounding procedures can be used to improve upon the computational complexity of this algorithm.
    \item Following Section~\ref{sec:alg}, we review the related work, and then contrast our algorithm with previous results from a theoretical and empirical perspective. Our experiments show that our theoretically superior algorithm is empirically competitive with previous algorithms on a range of problem scenarios.
\end{enumerate}

%In the next section, we review linear experimental design and the relevant quantities of the problem, and state the lower bound we prove for the transductive linear bandit problem.
%In Section~\ref{sec:alg}, we present our proposed algorithm and give an upper bound on the sample complexity of the algorithm that matches the lower bound up to log factors.
%Following Section~\ref{sec:alg}, we review the related work, and how they compare to our results.
%inally, given the context of the existing work, we perform experimental studies to compare the works.
\subsection{Notation}
For each $z\in \calZ$ define the \textit{gap} of $z$, $\Delta(z) = (z_{\ast} - z)^\top\theta^{\ast}$ and furthermore, $\Delta_{\min} = \min_{z \neq z_{\ast}} \Delta(z)$. If $A \in \mathbb{R}_{\geq 0}^{d\times d}$ is a positive semidefinite matrix, and $y\in \mathbb{R}^d$ is a vector, let $\|y\|^2_{A} := y^\top A y$ denote the induced semi-norm.
Let $\simplex_{\calX} := \{\lambda\in\mathbb{R}^{|\calX|}:\lambda\geq 0, \sum_{x\in\calX} \lambda_x = 1\}$ denote the set of probability distributions on $\calX$.
Taking $\calS \subset \calZ$ to a subset of the arm set, we define two operators we define $\calY(\calS) = \{z-z': \forall \ z, z' \in \calS, z \neq z'\}$ as the directions obtained from the differences between each pair of arms and $\calY^{\ast}(\calS) = \{z_{\ast} -z: \forall \ z \in \calS \setminus z_*\}$ as the directions obtained from the differences between the optimal arm and each suboptimal arm.
Finally, for an arbitrary set of vectors $\calV\subset \mathbb{R}^d$, define $\rho(\calV) = \min_{\lambda\in \simplex_{\calX}}\max_{v\in \calV} \|v\|^2_{(\sum_{x\in \calX} \lambda_x xx^\top)^{-1}}$ - this quantity will be crucial in the discussion of our sample complexity and is motivated in Section~\ref{sec:leastSquaresReview}

\section{Transductive Linear Bandits Problem}\label{sec:LinearExperimentalDesign}
Consider known finite collections of $d$-dimensional vectors $\calX\subset \mathbb{R}^d$ and $ \calZ\subset \mathbb{R}^d$ , known confidence $\delta\in(0,1)$, and unknown $\theta^{\ast}\in \mathbb{R}^d$.
The objective is to identify $z_{\ast} = \argmax_{z\in \calZ} z^{\top}\theta^{\ast}$ with probability at least $1-\delta$ while taking as few measurements in $\mc{X}$ as possible.
Formally, a transductive linear bandits algorithm is described by a \textbf{selection rule}
% tuple $(x_t, \tau, \widehat{i})$ where
$X_t \in \mc{X}$ at each time $t$ given the history $(X_s,R_s)_{s < t}$, \textbf{stopping time} $\tau$ with respect to the filtration $\mc{F}_t = (X_s,R_s)_{s\leq t}$, and \textbf{recommendation rule} $\widehat{z} \in \mc{Z}$ invoked at time $\tau$ which is $\mc{F}_{\tau}$-measurable.
We assume that $X_t$ is $\mc{F}_{t-1}$-measurable and may use additional sources of randomness; in addition at each time $t$ that $R_t= X_{t}^{\top}\theta^{\ast} + \eta_t$ where $\eta_t$ is independent, zero-mean, and 1-sub-Gaussian.
% At each time $t=1,2,\dots$ the player selects an index $x_t\in \calX$ and nature reveals observes a reward $r_t= x_{t}^{\top}\theta^{\ast} + \epsilon$ where the noise term $\epsilon$ is assumed independent, zero-mean, and 1-sub-Gaussian.
% If $\mc{F}_t = ( x_s, r_s )_{s=1}^t$ then $r_{t-1}$ and $x_t$ are $\mc{F}_{t-1}$-measurable.
Let $\mb{P}_{\theta^\ast},\mb{E}_{\theta^\ast}$ denote the probability law of $R_t | \mc{F}_{t-1}$ for all $t$.

\begin{definition}
We say that an algorithm for a transductive bandit problem is $\delta$-PAC for $\mc{X},\mc{Z} \subset \mb{R}^d$ if for all $\theta^\ast \in \mb{R}^d$ we have $\mb{P}_{\theta^\ast}( \widehat{z} = z_{\ast} ) \geq 1-\delta$.
% with probability at least $1-\delta$, for any $\theta^{\ast}$, $\tau$ is finite and $\widehat{z} = \argmax_{z\in \calZ } z^T\theta^{\ast}$.
\end{definition}

\subsection{Optimal allocations}
In this section we discuss a number of ways we can allocate a measurement budget to the different arms.
The following establishes a lower bound on the expected number of samples any $\delta$-PAC algorithm must take.
\begin{theorem}\label{thm:LowerBound}
Assume $\eta_t \overset{iid}{\sim} \calN(0,1)$ for all $t$. Then for any $\delta\in (0,1)$, any $\delta$-PAC algorithm must satisfy
\begin{equation*}
    \E_{\theta^\ast}[\tau] \geq \log(1/2.4\delta)\min_{\lambda\in \simplex_{\calX}}\max_{z\in \calZ\setminus \{z_{\ast}\}} \frac{\|z_{\ast} - z\|^2_{(\sum_{x\in \calX} \lambda_x x x^\top)^{-1}}}{((z_\ast-z)^\top\theta^{\ast})^2}.
\end{equation*}
\end{theorem}
This lower bound is proved in Appendix~\ref{sec:proofOfLowerBound} using standard techniques and employs the transportation inequality of \cite{kaufmann2016complexity}. It generalizes a previous lower bound in the setting of linear bandits \cite{soare2015sequential} and lower bounds in the combinatorial bandit literature \cite{chen2017nearly}.

\textbf{Optimal static allocation.} To demonstrate that this lower bound is tight, %recall $z_{\ast} = \arg\max_{z \in \calZ} z^\top \theta^*$ and
define
\begin{equation}\label{eqn:OptimalAllocation}
\lambda^{\ast} := \underset{\lambda\in \simplex_{\calX}}{\argmin}\max_{z\in \calZ\setminus\{z_{\ast}\}} \frac{\|z_{\ast} - z\|^2_{(\sum_{x \in \calX} \lambda_x x x^\top)^{-1}}}{((z_{\ast}-z)^{\top}\theta^{\ast})^2} \text{ and } \psi^{\ast} = \max_{\calZ\setminus\{z_{\ast}\}} \frac{\|z_{\ast} - z\|^2_{(\sum_{x \in \calX} \lambda_x^\ast x x^\top)^{-1}}}{((z_{\ast}-z)^{\top}\theta^{\ast})^2},
\end{equation}
where $\psi^{\ast}$ is the value of the lower bound and $\lambda^{\ast}$ is the allocation that achieves it. Suppose we sample arm $x \in \mathcal{X}$ exactly $2\lfloor \lambda^{\ast}_x N \rfloor$ times where we assume\footnote{Such an assumption is avoided by a sophisticated rounding procedure that we will describe shortly.} $N \in \mathbb{N}$ is sufficiently large so that $\min_{x : \lambda_x > 0} \lfloor \lambda_x N \rfloor > 0$.
If $N = \lceil 2 \psi^* \log(|\mc{Z}|/\delta) \rceil$ then as we will show shortly (Section~\ref{sec:leastSquaresReview}), the least squares estimator $\widehat{\theta}$ satisfies $(z_{\ast} - z)^\top \widehat{\theta} > 0$ for all $z \in \mc{Z} \setminus z_{\ast}$ with probability at least $1-\delta$.
Thus, with probability at least $1-\delta$, $z_{\ast}$ is equal to $\widehat{z} = \arg\max_{z \in \mc{Z}} z^\top \widehat{\theta}$ and the total number of samples is bounded by $2N$ which is within $4\log(|\mc{Z}|)$ of the lower bound.
Unfortunately, of course, the allocation $\lambda^*$ relies on knowledge of $\theta^*$ (which determines $z_*$) which is unknown a priori, and thus this is not a realizable strategy.

\textbf{Other static allocations.}
Short of $\lambda^*$ it is natural to consider allocations that arise from optimal linear experimental design \cite{pukelsheim2006optimal}.
For the special case of $\mc{X} = \mc{Z}$ it has been argued ad nauseam that a $G$-optimal design, $\argmin_{\lambda\in \simplex_{\calX}}\max_{x\in \calX, x\neq x_\ast} \|x\|^2_{(\sum_{x \in \calX} \lambda_x x x^\top)^{-1}}$,
is woefully loose since it does not utilize the differences $x-x'$, $x,x'\in \calX$  \cite{lattimore2016end, soare2014best, xu2018fully}.
Also for the $\mc{X} = \mc{Z}$ case, \cite{yu2006active,soare2014best}  have proposed the static
$\cal{X}\cal{Y}$-allocation given as $\argmin_{\lambda\in \simplex_{\calX}}\max_{x,x' \in \calX} \|x -x'\|^2_{(\sum_{x \in \calX} \lambda_x x x^\top)^{-1}}$.
In \cite{soare2014best} it is shown that no more than $O(\tfrac{d}{\Delta_{\min}^2} \log(|\calX|\log(1/\Delta_{\min})/\delta))$ samples from each of these allocations suffice to identify the best arm.
While the above discussion demonstrates that for every $\theta^\ast$ there exists an optimal static allocation (that explicitly uses $\theta^\ast$) that nearly achieves the lower bound, any fixed allocation with no prior knowledge of $\theta^\ast$ can require a factor of $d$ more samples.

\begin{proposition}\label{prop:fixedLowerBound}
    Let $c,c'$ be universal constants. For any $\gamma > 0$, $d$ even, there exists sets $\calX=\calZ \subset \mb{R}^d$ and a set $\Theta \subset \mathbb{R}^d$, such that $\inf_{\mc{A}}\max_{\theta\in \Theta} \E_\theta[\tau] \geq \frac{c d\log(1/\delta)}{\gamma}$ where $\mc{A}$ is the set of all algorithms that are $\delta$-PAC for $\mc{X},\mc{Z}$ and take a static allocation of samples.
    On the other hand $\psi^{\ast}/c'\leq  d+\frac{1}{\gamma}$ for every choice of $\theta^\ast \in \Theta$.
\end{proposition}

The proof of this proposition can be found in Appendix~\ref{sec:fixedLowerBound}.

\textbf{Adaptive allocations.}
As suggested by the problem definition, our strategy is to adapt our allocation over time, informed by the observations up to the current time.
Specifically, our algorithm will proceed in rounds where at round $t$, we perform an $\cal{X}\cal{Y}$-allocation that is sufficient to remove all arms $z \in \calZ$ that have gaps of at least $2^{-(t+1)}$.
We show that the total number of measurements accumulates to $\psi^\ast \log(|\calZ|/\delta)$ times some additional logarithmic factors, nearly achieving the optimal allocation as well as the lower bound.
In Section~\ref{sec:related}, we review other related procedures for the specific case of $\mc{X} = \mc{Z}$.

\subsection{Review of Least Squares}\label{sec:leastSquaresReview} Given a fixed design $\mathbf{x}_T = (x_t)_{t=1}^T$ with each $x_t\in \calX$ and associated rewards $(r_t)_{t=1}^T$, a natural approach is to construct the ordinary-least squares (OLS) estimate $\widehat\theta =  (\sum_{t=1}^T x_tx_t^{\top})^{-1} (\sum_{t=1}^T r_tx_t)$.
One can show $\widehat{\theta}$ is unbiased with covariance $\preceq (\sum_{t=1}^T x_tx_t^{\top})^{-1}$.
Moreover, for any $y \in \mathbb{R}^d$, we have\footnote{There is a technical issue of whether the set $\calZ$ lies in the span of $\calX$ which in general is necessary to obtain unbiased estimates of $(z_{\ast} - z)^\top \theta^*$. Throughout the following we
    assume that $\text{span}({\calX}) =\mathbb{R}^d$.}
\begin{equation}\label{eqn:ConfidenceSet}
\mb{P}\left( y^{\top}(\theta^{\ast} - \widehat{\theta}) \geq
\sqrt{\|y\|^2_{(\sum_{t=1}^Tx_{t}x_{t}^\top)^{-1}}2 \log(1/\delta)} \right) \leq \delta.
\end{equation}

In particular, if we want this to hold for all $y\in \calY^{\ast}(\calZ)$, we
need to union bound over $\calZ$ replacing $\delta$ with  $\delta/|\calZ|$. Let us now use this to analyze the procedure discussed above (in the discussion on the optimal static allocation after Theorem~\ref{thm:LowerBound}) that gives an allocation matching the lower bound.
With the choice of $N=\lceil 2 \psi^* \log(|\mc{Z}|/\delta) \rceil$ and the
allocation $2\lfloor\lambda_x^{\ast}N\rfloor$ for each $x \in \calX$, we have
for each $z \in \calZ \setminus z_{\ast}$ that with probability at least
$1-\delta$,
\begin{equation*}
(z_{\ast} - z)^\top \widehat{\theta} \geq (z_{\ast} - z)^\top \theta^\ast - \sqrt{\|z_{\ast} - z\|^2_{(\sum_{x} 2 \lfloor N \lambda^{\ast}_x xx^T \rfloor )^{-1}}2 \log(|\mc{Z}|/\delta)}\geq 0
\end{equation*}
since for each $y=z_{\ast} - z\in \calY^{\ast}(\calZ)$ we have
\begin{equation} \label{eqn:verify}
y^\top\Big(\sum_{x\in \calX} 2 \lfloor N \lambda^{\ast}_x \rfloor x x^\top\Big)^{-1}y\leq y^\top\Big(\sum_{x\in \calX} \lambda^{\ast}_x x x^\top\Big)^{-1} y/N\leq ((z_{\ast}-z)^\top \theta^{\ast})^2/(2 \log(|\mc{Z}|/\delta) ),
\end{equation}
% Plugging this into our confidence interval, we see that for each $z_{\ast} - z\in \calY^{\ast}(\calZ)$, $(z_{\ast} - z)^\top (\theta^{\ast} - \widehat{\theta}) \leq \Delta(z)$ which implies $(z_{\ast} - z)^\top  \widehat{\theta} > 0 $ for all $x\in \calZ$,
where the last inequality plugs in the value of $N$ and the definition of $\psi^{\ast}$. The fact that at most one $z' \in \calZ$ can satisfy $(z'-z)^\top \widehat{\theta} > 0$ for all $z\neq z'\in \calZ$, and that $z'=z_{\ast}$ does,
certifies that $\widehat{z} = \arg\max_{z \in \calZ} z^\top \widehat{\theta}$ is indeed the best arm with probability at least $1-\delta$.
Note that equation~\eqref{eqn:verify} provides the motivation for how the form of $\psi^{\ast}$ is obtained.
Rearranging, it is equivalent to,
\[N\geq 2\log(|\calZ|/\delta)\max_{\calZ\setminus\{z_{\ast}\}} \frac{\|z_{\ast} - z\|^2_{(\sum_{x\in \calX} \lambda^{\ast}_x  x x^\top)^{-1}}}{((z_{\ast}-z)^{\top}\theta^{\ast})^2} \text{ for all } z\in \calZ\setminus \{z_{\ast}\}\]
Thinking of the right hand side of the inequality as a function of $\lambda$, $\lambda^{\ast}$ is precisely chosen to minimize this quantity and hence the sample complexity.

\subsection{Rounding Procedures}
We briefly digress to address a technical issue. Given an allocation $\lambda$ and an arbitrary subset of vectors $\calY$, in general, drawing $N$ samples $\mathbf{x}_{N} := \{x_1,\dots,x_N\}$ at random from $\calX$ according to the distribution $\lambda_x$ may result in a design where  $\max_{y\in \calY}\|y\|_{(\sum_{t=1}^N x_t x_t^\top)^{-1}}^2$ (which appears in the width of the confidence interval~\eqref{eqn:ConfidenceSet})
differs significantly from $\max_{y\in \calY} \|y\|_{(\sum_{x\in \calX} \lambda_x x x^\top)^{-1}}^2/N$.
%By definition, $y\in \calY(\calzhat)$, $y^\top(\sum_{x\in \calX} N\lambda^{\ast}_x xx^{\top})^{-1}y\leq \psi^{\ast}/N$,
Naive strategies for choosing $\mathbf{x}_{N}$ will fail. We can not simply use an allocation of $N\lambda_x$ samples for any specific $x$ since this may not be an integer.
Furthermore, greedily rounding $N\lambda_x$ to an allocation $\lfloor N\lambda_x\rfloor$ or $\lceil N\lambda_x\rceil$ may result in too few than necessary, or far more than $N$ total samples if the support of $\lambda$ is large. %In general simply drawing $N$ samples $x_1, \cdots, x_N$ will not result in a design satisfying $\max_{y\in \calY}\|y\|_{(\sum_{t=1}^N x_t x_t^\top)^{-1}}^2$.
However, given $\epsilon>0$, there are \emph{efficient rounding procedures}
that produce $(1+\epsilon)$-approximations as long as $N$ is greater than some
minimum number of samples $r(\epsilon)$. In short, given $\lambda$ and a choice
of $N$ they return an allocation $\mathbf{x}_N$ satisfying $\max_{y\in \calY}
\|y\|^2_{(\sum_{i=1}^N  x_ix_i^{\top})^{-1}} \leq (1+\epsilon)\max_{y\in \calY}
\|y\|_{(\sum_{x\in \calX} \lambda_x x x^\top)^{-1}}^2/N$.
Such a procedure with $r(\epsilon) \leq O(d/\epsilon^2)$ is described in Section~\ref{sec:rounding} in the supplementary.
In our experiments we use a rounding procedure from \cite{pukelsheim2006optimal} that is easier to implement (also see \cite[Appendix C]{soare2014best}) with $r(\epsilon) = (d(d+1)/2+1)/\epsilon$. In general $\epsilon$ should be thought of as a constant, i.e. $\epsilon = 1/5$. The number of samples $N$ we need to take in our algorithm will be significantly larger than $5d^2$, so the impact of the rounding procedure is minimal.

 %This suggests that we should  if we are interested in a specific direction $y\in \mathbb{R}^d$, the confidence interval suggests that we should choose our design to minimize our variance in $y$, i.e. in other words choose our deesign as $\argmin_{\mathbf{x}^T} y^{T} A_{\mathbf{x}_T} y$.
%\subsection{Notation} % We let $\Sigma_n = \{\lambda\in \mathbb{R}^n:\sum_{i=1}^n \lambda_i = 1\}$ be the simplex of probability vectors in $d$ dimensions.
%The sub-optimality gap of an arm $z\in \calZ$ is denoted as $\Delta(z) =(z_{\ast}-z)^{\top}\theta^{\ast}$.
%Taking $\calS \subset \calZ$ to a subset of the arm set, we define two operators on the arm subset.
%We denote $\calY(\calS) = \{z-z', \forall \ z, z' \in \calS, z \neq z'\}$ as the directions obtained from the differences between each pair of arms and $\calY^{\ast}(\calS) = \{z_{\ast} -z, \forall \ z \in \calS \setminus z_*\}$ as the directions obtained from the differences between the optimal arm and each suboptimal arm. Finally, for $\calS\subset \calZ$ denote $\rho(\calS) :=  \min_{\lambda\in \Sigma_{|\calX|}} \max_{s\in \calS} \|s\|_{(\sum_x \lambda_x xx^{\top})^{-1}}^2$.

\section{Sequential Experimental Design for Transductive Linear Bandits}\label{sec:alg}
Our algorithm for the pure exploration transductive bandit is presented in Algorithm~\ref{alg:main}. The algorithm proceeds in rounds, keeping track of the active arms $\calzhat_t \subseteq \calZ$ in each round $t$.
At the start of round $t-1$, it samples in such a way to remove all arms with gaps greater than $2^{-t}$.
Thus denoting $\calS_t := \{z\in \calZ: \Delta(z)\leq 2^{-t}\}$, in round $t$ we expect $\calzhat_t\subset \calS_t.$

As described above, if we knew $\theta^{\ast}$, we would sample according to the optimal allocation $\argmin_{\lambda\in \simplex_{\calX}}\max_{z\in \widehat{\calZ}_t} \|z_{\ast} - z\|^2_{(\sum_{x\in \calX} \lambda_x xx^\top)^{-1}}/((z_{\ast}-z)^{\top}\theta^{\ast})^2$.
However, instead at the start of round $t$, if we simply have an upper bound on the gaps, $\Delta(z)\leq 2^{-t}$ and we do not know the best arm, we can lower-bound the above objective by
$(2^{t})^{2}\min_{\lambda \in \simplex_{\calX}}\max_{y\in \calY(\calzhat_t)} \|y\|^2_{(\sum_{x\in \calX} \lambda_x xx^T)^{-1}}$ \footnote{
 Where we recall for any subset $\calS\subset \calZ, \calY(\calS) := \{z-z':z,z'\in \calS\}$ and for an arbitrary subset $\calV\subset\mathbb{R}^d$ we have~$\rho(V) = \min_{\lambda\in \simplex_{\calX}}\max_{v\in \calV} \|v\|^2_{(\sum_{x\in \calX} \lambda_x xx^\top)^{-1}}$.}.
 This motivates our choice of $\lambda_t$ and $\rho(\calY(\calzhat_t))$.
%Motivated by the confidence bound \eqref{eqn:ConfidenceSet}, we sample according to the allocation
%$\lambda_t^{\ast} = \arg\min_{\lambda\in \simplex_{\calX}}\max_{y\in \calY(\calzhat_t)} \|y\|_{(\sum_{x\in \calX} \lambda_x xx^{\top})^{-1}}^2$,
%where for any subset $S\subset \calZ$,
%Furthermore, for any set $\calY$ define $\rho(\calY) := \min_{\design \in \simplex_{^{\calX}}}\max_{\yarm \in \calY}\|\yarm\|_{(\sum_{x\in \calX} \lambda_x xx^{\top})^{-1}}^2$.
%As discussed previously, we would like to sample arms in $\calX$ according to a distribution that minimizes the maximum variance over just the directions in $\calY^{\ast}(\calzhat)$, meaning the allocation corresponding to $\min_{\lambda\in \Sigma^{n}}\max_{z\in \calzhat} \|z_{\ast} - z\|_{(\sum_{x\in X} \lambda_x xx^{\top})^{-1}}^2 $. Since we do not know which arm is the optimal one, we instead construct the distribution that minimizes the maximum variance over the set of all possible uncertain directions in round $t$, i.e. $\lambda^{\ast}_T=\argmin_{\lambda\in \Sigma^{|\calX|}}\max_{z,z'\in \calzhat} \|z - z'\|_{(\sum_{x\in X} \lambda_x xx^{\top})^{-1}}^2$.
%Recalling that we can build a confidence interval on each $y\in \calY$ is of the form $\sqrt{\|y\|^2_{} \log(K^2/\delta)}$
% Now, the  design $\lambda_t^{\ast}$ defined in the algorithm is the optimal sampling distribution over $\calX$ if we allow ourselves to take \emph{infinitely} many samples.
Thus by the same logic used in Section~\ref{sec:leastSquaresReview}, $N_t = \lceil 8(2^{t+1})^2(1+\epsilon)\rho(\calY(\calzhat_t))\log(|\calZ|^2/\delta_t)\rceil$ samples should suffice to guarantee that we can construct a confidence interval on each $(z-z')^{\top}\theta^{\ast}$
for $(z-z')\in \calY(\calzhat_t)$ of size at most $2^{-(t+1)}$
(with the $|\calZ|^2$ in the logarithm accounting for a union bound over arms).
The $(1+\epsilon)$ accounts for slack from the rounding principle.
Finally, we remove an arm $z$ if there exists an arm $z'$ so that the empirical gap $(z'-z)^\top\widehat\theta_t > 2^{-(t+2)}$.

\begin{algorithmfloat}%{\textwidth}\refstepcounter{alg}\label{alg:main}
\centering
\fbox{\small
\begin{minipage}{\textwidth}\refstepcounter{alg}\label{alg:main}
\textbf{Algorithm~ 1}: \textbf{RAGE}($\calX,\calZ,\delta$): \textbf{R}andomized \textbf{A}daptive \textbf{G}ap \textbf{E}limination
\hrule
\vspace{2mm}
\textbf{Input}: Arm set $\xset$, rounding approximation factor $\epsilon$ with default value $1/5$, minimum number of samples needed to obtain rounding approximation $r(\epsilon)$, and confidence level $\delta \in (0, 1)$. \\
Let $\zhatset_1 \gets \calZ, t\gets 1$ \\
\textbf{while} $|\zhatset_{t}|>1$ \textbf{do}\\
\begin{quoting}[vskip=0pt, leftmargin=10pt]
    $\delta_{t} \gets \tfrac{\delta}{t^2}$ \\
    $\design_{t}^{\ast} \gets \arg\min_{\design \in \simplex_{\calX}}\max_{\yarm \in \yset(\calzhat_{t})}\|\yarm\|_{(\sum_{x\in \calX} \lambda_x xx^{\top})^{-1}}^2$ \\
    $\rho(\yset(\calzhat_{t})) \gets \min_{\design \in \simplex_{^{\calX}}}\max_{\yarm \in \yset(\calzhat_{t})}\|\yarm\|_{(\sum_{x\in \calX} \lambda_x xx^{\top})^{-1}}^2$ \\
    $N_t \gets \max\big\{\big\lceil 8 (2^{t+1})^2 \rho(\yset(\calzhat_{t})) (1 + \varepsilon) \log(|\calZ|^2/\delta_{t}) \big\rceil, r(\epsilon)\big\}$ \\
    $\armseq_{N_t} \gets \text{\textsc{Round}}(\design^{\ast}, N_t)$ \\
    Pull arms $x_1,\dots, x_{N_t}$ and obtain rewards $r_1, \dots, r_{N_t}$ \\
    Compute $\hatparam_t= A_{t}^{-1}b_{t}$ using $A_{t} := \sum_{j=1}^{N_t}\arm_{j}\arm_{j}^{\top}$ and $b_{t} := \sum_{j=1}^{N_t}\arm_{j}\reward_{j}$ \\
    $\calzhat_{t+1} \gets \calzhat_t \setminus\big\{z \in \calzhat|\exists \ z' \in \calzhat: 2^{-(t+2)} \leq (z'-z)^{\top}\widehat{\theta}_{t}\big\}$\\
    %$\xhatset_{t+1} \gets \text{Transductive-GapElim}(\xset, \xhatset_{t}, \design_{t}^{\ast}, \randnumsamples_{t}, \delta_{t})$ \\
    $\phaseindex \gets \phaseindex + 1$ \\
\end{quoting}
\textbf{Output}: $\calzhat_{t+1}$
\end{minipage}
}
\end{algorithmfloat}

\begin{theorem}\label{thm:SampleComplexity}
With probability greater than $1-\delta$, using an $\epsilon$-efficient rounding procedure, Algorithm~\ref{alg:main} correctly identifies the optimal arm $z_*$ and requires a worst-case sample complexity
\[\randnumsamples \leq \sum_{\phaseindex=1}^{\lfloor \log_2(1/\Delta_{\min}) \rfloor} \max\big\{8 \big\lceil (2^{\phaseindex+1})^2 \rho(\yset(\sset_{\phaseindex}))(1+ \epsilon)\log(\phaseindex^2|\calZ|^2/\delta)\big\rceil, r(\epsilon)\big\}\]
where $\calS_t = \{z\in \calZ: \Delta(z)\leq 2^{-t}\}$. In particular, \textsc{Round} can be
chosen so that $r(\epsilon) = O(d/\epsilon^2)$. Furthermore, $N \leq
c \psi^{\ast}\log(1/\Delta_{\min})\log(|\calZ|^2\log(1/\Delta_{\min})^2/\delta)+r(\epsilon)\log_2(1/\Delta_{\min})$ for some absolute constant $c$, in other words Algorithm \ref{alg:main} is instance optimal up to logarithmic factors.
\end{theorem}

We provide a proof of the sample complexity bound in Section~\ref{sec:SampleComplexityProof}.

\subsection{Interpreting the sample complexity. }
Up to logarithmic factors, Algorithm~\ref{alg:main} matches the lower bound obtained in Theorem~\ref{thm:LowerBound}.  However, the term $\rho(\calY(\calS_t))$ may seem a bit mysterious. In this section we try to interpret this quantity in terms of the geometry of $\calX$ and $\calZ$.

Let $\conv(\calX\cup -\calX)$ denote the convex hull of $\calX\cup -\calX$, and for any set $\calY\subset \mb{R}^d$ define the gauge of $\calY$, \[\gamma_{\calY} = \max\{c>0: c\calY\subset \conv(\calX\cup -\calX)\}.\]
In the case where $\calY$ is a singleton $\calY = \{y\}$, $\gamma(y):=\gamma_{\calY}$ is the \emph{gauge norm} of $y$ with respect to $\conv(\calX\cup -\calX)$, a familiar quantity from convex analysis \cite{rockafellar2015convex}. We can provide a natural upper bound for $\rho(\calY)$ in terms of the gauge.
\begin{lemma}\label{lem:interpLemma} Let $\calY\subset\mb{R}^d$. Then
    \begin{equation}\label{eqn:BoundsOnRho}
    \max_{y\in \calY} \|y\|_2^2/( \max_{x\in\calX} \|x\|_2) \leq \rho(\calY) \leq d/\gamma_{\calY}^2.
    \end{equation}
    In the case of a singleton $\calY = \{y\}$, we can improve the upper bound to $\rho(\calY) \leq 1/\gamma(y)^2$.
\end{lemma}

%{\color{red} [The above feels a lot more like math than interpretability. Maybe a picture of $\calX,\calZ$ and some ellipse?]}
%\textcolor{blue}{I also think a picture would be nice here to convey ideas.}
{%\color{red} [In what follows I think you're trying to convince me that $r(\calY(S_t)$ could be substantially less than $4$, which is used in previous analyses. And that both inequalities are achievable, but its getting lost...]}
The proof of this Lemma is in Appendix~\ref{sec:interpLemma}.
To see the potential for adaptive gains we focus on the case of linear bandits where $\calX = \calZ$. Consider an example with $\calX = \{e_{i}\}_{i=1}^d\cup \{z'\}$ for $z' = (\cos(\alpha), \sin(\alpha),0,\cdots, 0)$ where $\alpha\in [0,.1)$, and $\theta^{\ast} = e_1$.
Note that $\Delta_{\min} \approx \sin(\alpha) \approx \alpha$.
Then $\calS_1 = \calX$, and an easy comptation shows $\gamma_{\calY(\calX)}\leq 2$.
After the first round, all arms except $e_1$ and $z'$ will be removed, so $\calY(\calS_t) = \{e_1 - z'\}$ for $t\geq 2$, and $\gamma_{\calY(\calS_t)} \approx 1/\sin(\alpha)\approx 1/\alpha$. Summing over all rounds, we see that this implies a sample complexity of
$O(d\log(\log(1/\alpha) d^2/\delta))$ which up to $\log$ factors is independent of the gap and a significant improvement over the static $\calX \calY$-allocation sample complexity of $d/\alpha^2$.
%In the specific case of linear bandits, note that $\max_{x,x'\in \calS_t}\|x -
%x'\|_{(\sum \lambda_x xx^{\top})^{-1}}^2 \leq 4 \max_{x\in \calX} \|x\|_{(\sum
%    \lambda_x xx^{\top})^{-1}}^2\leq 4d$. Hence, using the Kiefer-Wolfowitz
%    theorem shows that $\rho(\calY(S_t))\leq 4d$ for $S\subset\calX$, i.e. Lemma 3 \cite{soare2014best}. In particular, summing over all rounds gives a sample complexity
%\[\sum_{t=1}^{\log_2(1/\Delta_{\min})} (2^{t})^2\rho(\calY(\calS_t))\log(|S_t|t^2/\delta) \leq O\Big(\frac{d}{\Delta_{\min}^2}\log(\log(1/\Delta_{\min})K^2/\delta)\Big)\]
%matching the passive sample complexity obtained through following a static G-optimal design allocation from the start (see Theorem 1 in \cite{soare2014best}).

%\textcolor{red}{Consider a situation where $\calX$ is roughly uniformly distributed on the sphere and $\theta^{\ast} = x_{\ast}\in \calX$.
%In latter rounds, points with small gaps will be close to $x_{\ast}$ and so for $y\in \calY(S_t)$,  $\|y\|_2\approx 2^{-t}$.
%In particular, $r(\calY(S_t)) \approx 1/\max_{y\in \calY(S_t)}\|y\|_2$ and so
%using~\eqref{eqn:BoundsOnRho}, $\rho(\calY(S_t)) \approx d \max_{y\in \calYß(S_t)}\approx d 2^{-t},$
%which is significantly better than the naive bound of $4d$.}

%\textcolor{red}{In the case of one vector we can even drop the $d$, maybe worth putting in.}

\section{Related Work}\label{sec:related}
When $\calX = \calZ =  \{e_1,\cdots, e_d\}\subset \mathbb{R}^d$ is the set of standard basis vectors, the problem reduces to that of the best-arm identification problem for multi-armed bandits  which has been extensively studied
\cite{even2006action, jamieson2014lil, karnin2013almost,kaufmann2016complexity, chen2015optimal}.
In addition, pure exploration for combinatorial bandits where $\calX = \{e_1,\cdots, e_d\}\subset \mathbb{R}^d$ and $\calZ \subset \{0,1\}^d$ has also received a great deal of attention \cite{chen2017nearly, cao2017disagreement, chen2014combinatorial, chen2016pure}.

In the setting of linear bandits when $\calX = \calZ$, despite a great deal of work in the regret and contextual settings
\cite{abbasi2011improved, li2010contextual, lattimore2016end, dani2008stochastic}, there has been far less work on linear bandits
for pure exploration. This problem was first introduced in \cite{soare2014best} and since then, there have been a few other works on this topic, \cite{tao2018best,karnin2016verification,xu2018fully} that we now discuss.
% \textcolor{red}{What does Hoffman \cite{hoffman2014correlation} do?}
\begin{itemize}[topsep=0pt, partopsep=0pt, labelindent=0pt, labelwidth=0pt, leftmargin=10pt]
\item Soare et al.~\cite{soare2014best} made the initial connections to G-optimal
    experimental design. That work provides the first passive algorithm with a sample complexity of $O(\tfrac{d}{\Delta_{\min}^2}\log(|\calX|/\delta) + d^2)$.
Note that the $d^2$ comes from the minimum number of samples needed for an efficient rounding procedure and thus could be reduced to $d$ using improved rounding procedures (see section~\cite{allen2017near}).
In addition to the passive algorithm, they provide an adaptive algorithm, $\calX\calY$-adaptive algorithm for linear bandits.
Their algorithm is very similar to ours, with two notable differences.
Firstly, instead of using an efficient rounding procedure, they use a greedy iterative scheme to compute an optimal allocation.
Secondly, their algorithm does not discard items that are provably sub-optimal.
As a result, their sample complexity (up to logarithmic factors) scales as $\max\{M^{\ast}, \psi^{\ast}\}\log(|\calX|/(\Delta_{\min}\delta)) +d^2$ where $M^{\ast}$ is defined (informally) as the amount of samples needed using a static allocation to remove all sub-optimal directions in $\calY(\calX)\setminus\calY^{\ast}(\calX)$.

\item In Tao et al.~\cite{tao2018best}, the focus is on developing different
    estimators with the goal of removing the constant term $d^2$ in Soare et al.'s passive sample complexity.
Instead of using a rounding procedure, they use a different estimator than the OLS estimator $\theta^{\ast}$.
Note that the rounding procedure in \cite{allen2017near} and described in the supplementary could have been applied directly to Soare's static allocation algorithm giving the same sample complexity as the one obtained in~\cite{tao2018best}.
They also provide an adaptive algorithm \emph{ALBA}, that achieves a sample complexity of $O(\sum_{i=1}^d 1/\Delta_{i}^2)$ where $i$ is the $i$-th smallest gap of the vectors in $\calX$.
It is easy to see that this sample complexity is not optimal: imagine a situation in which the vectors of $\calX$ with the $(d-1)$-smallest gaps are identical to the vector $x' \neq x^\star$. Then we only need to pay once for the samples needed to remove $x'$, not $(d-1)$-times. %If there is a vector $z'\in\calZ$ roughly in the direction of $z_{\ast} - z$, then we .
Finally, their algorithms do not compute the optimal allocation over differences
of vectors in $\calX$, but instead on $\calX$ directly \`{a} la G-optimal design.
We will see the inefficiency of this strategy in the experiments.

\item Karnin \cite{karnin2016verification} provides an algorithm that uses
    repeated rounds (for probability amplification) of exploration phases
    combined with verification phases to provide an asymptotically optimal algorithm, meaning when
    $\delta\rightarrow 0$ the sample complexity divided by $\log(1/\delta)$
    approaches $\psi^{\ast}$. Though this is a nice theoretical result, the
    algorithm is not practical; the exploration phase is simply a na\"{i}ve passive $G$-optimal design. %so it is not an algorithm you would actually run. expected vs high probability - minor comment. Tanners point of exploration vs verification - what we find out, these things are actually the same- he and tao pose open questions that we answer up to logarithmic factors

\item In Xu et al.~\cite{xu2018fully}, a fully adaptive algorithm, inspired by
    the UGapE algorithm~\cite{gabillon2012best}, LinGapE is proposed. Since LinGapE is fully adaptive, a confidence bound allowing for dependence in the samples is necessary and the authors employ the self-normalized bound of~\cite{abbasi2011improved}. %This confidence interval results in an extra factor of $O(\sqrt{d})$ in their sample complexity. The authors claim that this bound is not tight, but nonetheless we observe scaling of this nature in our simulations.
    The algorithm requires each arm to be pulled once - an undesirable characteristic of a linear bandit algorithm since the structure of the problem allows for information to be obtained about arms that are not pulled. A recent work \cite{kazerouni2019best}, extends this algorithm to the setting of generalized linear models where the expected reward of pulling arm $z$ reward is given by a noisy ovservation of a non-linear link function of $z^{\top}\theta^{\ast}$.
\end{itemize}
Finally, we mention \cite{yu2006active}, which considers transductive experimental design from a computational and optimization perspective, and explores $\calX \calY$-allocation for arbitrary kernels.

%Teh problem of general transductive experimental design has been studied in blah and blah. GENERAL TRANSDUCTIVE EXPERIMENTAL DESIGN: , .

\section{Experiments}
In this section, we present simulations for the linear bandit pure exploration problem and the general transductive bandit problem.
We compare our proposed algorithm with both adaptive and non-adaptive strategies.
The adaptive strategies are $\cal{XY}$-Adaptive allocation from \cite{soare2014best}, LinGapE from \cite{xu2018fully}, and ALBA from~\cite{tao2018best}, and the non-adaptive strategies are static $\cal{XY}$-allocation, as described in Section~\ref{sec:LinearExperimentalDesign}, and an oracle strategy that knows $\theta^{\ast}$ and samples according to $\lambda^{\ast}$.
%Precise
%implementations are given in the supplementary material, Section~\ref{sec:}.
%with existing algorithms discussed in Section~\ref{sec:related} along with a static $\calX\calY$ allocation and an oracle $\calX\calY$ allocation that knows the optimal arm and the gaps.
We do not compare to the algorithm given in~\cite{karnin2016verification} since it is primarily a theoretical contribution and in moderate-confidence regimes obtains only the non-adaptive sample complexity.
We run each algorithm at a confidence level of $\delta=0.05$. The empirical failure probability of each of the algorithms in the presented simulations is zero.
% We report the mean sample complexity of each algorithm over 20 runs and the confidence bars in each figure represent the standard error in the mean sample complexities.
To compute the samples for RAGE, we first used the Frank-Wolfe algorithm (with a precise stopping condition in the supplementary) to find $\lambda_t$, and then a rounding procedure from \cite{pukelsheim2006optimal} with various values of $\epsilon < 1/3$. Further implementation details of RAGE and discussion pertaining to the implementation of the other algorithms can be found in the supplementary material in Section~\ref{sec:implementation_details}.
We remark here that in our implementation of the $\calX\calY$-Adaptive allocation, we follow the experiments in~\cite{soare2014best} and allow for provably suboptimal arms to be discarded (though this is not how the algorithm is written in their paper). The resulting algorithm is then similar to our algorithm.

\textbf{Linear bandits: benchmark example.} The first experiment we present has become a benchmark in the linear bandit pure exploration literature since it was introduced in~\cite{soare2014best}.
%It is intended to demonstrate a problem instance where an adaptive allocation can significantly outperform a static allocation.
In this problem, $\calX = \{e_1, \ldots, e_d, x'\} \subset \mb{R}^d$
where $e_i$ is the $i$-th standard basis vector,
$x'= \cos(.01)e_1 + \sin(.01)e_2$, and $\theta^{\ast}= 2e_1$ so that $x_{\ast} = x_1$.
An efficient sampling strategy for this problem needs to focus on reducing uncertainty in the direction $(x_1-x_{d+1})$, which can be achieved by focusing pulls on arm $x_2 = e_2$ since it is most aligned with this direction.

% The minimum gap in this example is on the order of $10^{-5}$.
% Consequently, any algorithm with a sample complexity scaling directly with the inverse of the minimum gap squared could require as many as $10^{10}$ samples. The algorithm we propose requires significantly fewer samples to identify the optimal arm since the allocation strategy effectively hones in on the direction between the optimal arm and the second best arm.
% Following an initial round in which arms are sampled near uniformly, only arms $x_1$ and $x_{d+1}$ remain in the active arm set $\widehat{\calX}$ with high probability. Since $x_1$ and $x_{d+1}$ are nearly aligned, $\rho(\calY(\widehat{\calX}))$ drops to be nearly on the order of the minimum gap and plays a fundamental role in controlling the sample complexity of each ensuing phase.

The results for this experiment are shown in Fig.~\ref{fig:sincos}.
The proposed RAGE algorithm performs competitively with existing algorithms and
the oracle allocation. The $\calX\calY$-Adaptive algorithm is similar to RAGE, but with weaker theoretical guarantees, so naturally it performs nearly equivalently.
The LinGapE algorithm performs well when the number of dimensions and arms is small.
%However, the sample complexity scales with the dimension owing to a
%$\mathcal{O}\big(\sqrt{d}\big)$ penalty for using a confidence interval that
%does not require independence in the samples. This behavior can be observed
%in~Fig.~\ref{fig:sincos}.
However, as the number of arms grows, LinGapE will fall behind the competing algorithms. ALBA performs the worst of the recently proposed algorithms and this is to be expected since it computes an allocation on the $\calX$ set instead of on the $\calY(\cal{X})$ set.
% Once highly suboptimal arms are discarded, this set consists of the arms $x_1, x_2$ and $x_{d+1}$.
% Computing an allocation on this arm set directly instead of on the directions between the arms is wasteful since a significant amount of energy is spent pulling the optimal arm $x_1$ even though the most information can be obtained pulling arm $x_2$.
% As a result, ALBA performs similarly to a passive allocation.
% \textbf{Uniform sphere.}
% Following~\cite{tao2018best}, we consider a set $\calX \subset \mc{R}^{d}$ of arms in which we sweep the cardinality of $\calX$ and each arm is randomly sampled from the unit sphere centered at the origin. We select the two closet arms, $x, x' \in \calX$ and set $\theta = x+ \gamma(x'-x)$ with $\gamma = 0.01$ so that $x$ is the best arm.
This example clearly highlights the gains of adaptive sampling over non-adaptive allocations like static $\cal{XY}$-allocation. However, since $\calX$ is relatively small in this case, it fails to tease out important differences between the algorithms that can greatly increase the sample complexity. We construct examples to demonstrate these claims now.

\textbf{Many arms with moderate gaps.}
% The structure present in the linear bandit problem plays a key role in the necessary sample complexity since information about the reward of an arm can be obtained even when it is not selected.
% Indeed, the ability to obtain information about the reward of every arm when an arm is pulled can significantly reduce the number of samples needed to identify the optimal arm.
% As a result, an optimal pure exploration linear bandit algorithm should not give a sample complexity scaling with the number of arms.
In this example, for a given value of $n\geq 3$, we construct a set of arms $\calX \subset \mb{R}^2$, where
$\calX = \{e_1, \cos(3\pi/4)e_1 + \sin(3\pi/4)e_2\}\cup \{\cos(\pi/4 +\phi_i)e_1 + \sin(\pi/4 + \phi_i)e_2\}_{i=3}^{n}$ with $\phi_i\sim \calN(0,.09)$ for each $i\in \{3, \ldots, n\}$.
The parameter vector is fixed to be $\theta^{\ast} = e_1$ so that $x_1$ is the optimal arm, $x_2$ gives the most information to identify the optimal arm, and the remaining arms roughly point in the same direction with an expected gap of $\Delta\approx 0.3$.

In Fig.~\ref{fig:duplicate}, we show the results of the experiment as we increase the number of arms $n$.
We focus on the LinGapE algorithm in this example to demonstrate a linear scaling in the number of arms that occurs since LinGapE samples each arm once.
An efficient sampling strategy should focus energy on $x_2$, and as it does so, it will gain information about the arms that are nearly duplicates of each other, which is how RAGE performs.
%IWe can observe that the sample complexity of RAGE is nearly fixed as the number of arms grows, while the sample complexity of LinGapE scales linearly with the number of dimensions.
%This is to be expected as the LinGapE algorithm must pull each arm once in an initialization procedure.
%However, as demonstrated, pulling each arm in linear bandits is not necessary.\

\textbf{Uniform Distribution on a Sphere.} In this example, $\calX$ is sampled from a unit sphere of dimension $d=5$ centered at the origin. Following~\cite{tao2018best}, we select the two closest arms $x, x' \in \calX$ and let $\theta^{\ast}= x + \alpha (x'-x)$ where $\alpha = 0.01$. In Fig.~\ref{fig:uniform}, we show the sample complexity of the RAGE and ALBA algorithms as the number of arms is increased.
The RAGE algorithm significantly outperforms ALBA and this is primarily due to the fact that ALBA computes a G-optimal design on the active vectors in each round instead of on the differences between these vectors. Thus the ALBA sampling distribution can be focused on a very different set of arms from the optimal one.
%The distribution of arms lends itself naturally to an adaptive allocation since information can be obtained in any direction of interest as the number of arms grows.

\textbf{Transductive example.}
To conclude our experiments, we present a general transductive bandit example.
Since the existing algorithms in the linear bandit literature do not generalize to this problem, we compare with a static $\calX\calY$-allocation on $\calX, \calY(\calZ)$ and an oracle $\calX\calY$-allocation on $\calX, \calY^{\ast}(\calZ)$ that knows the optimal arm and the gaps.
We construct an example in $\mathbb{R}^d$ with $d$ even where $\calX = \{e_1, \ldots, e_d\}$.
The set $\calZ$ is also chosen so $|\calZ| = d$, the first $d/2$ vectors are given by ${z_1,\dots, z_{d/2}} =(e_1, \dots, e_{d/2})$ and then  $z_{d/2+j} = \cos(.1)e_{j}+ \sin(.1)e_{j+ d/2}$ for each $j\in\{1,\dots,d/2\}$.
Take $\theta^{\ast}=e_1$ so $z_1$ is the optimal arm.
The results of this simulation are depicted in Fig.~\ref{fig:transductive}.
The RAGE algorithm significantly outperforms the static allocation.
\renewcommand{\figurename}{Figure}
\begin{figure}[t]
    \centering
    \subfloat[][Benchmark]{\includegraphics[width=0.5\linewidth]{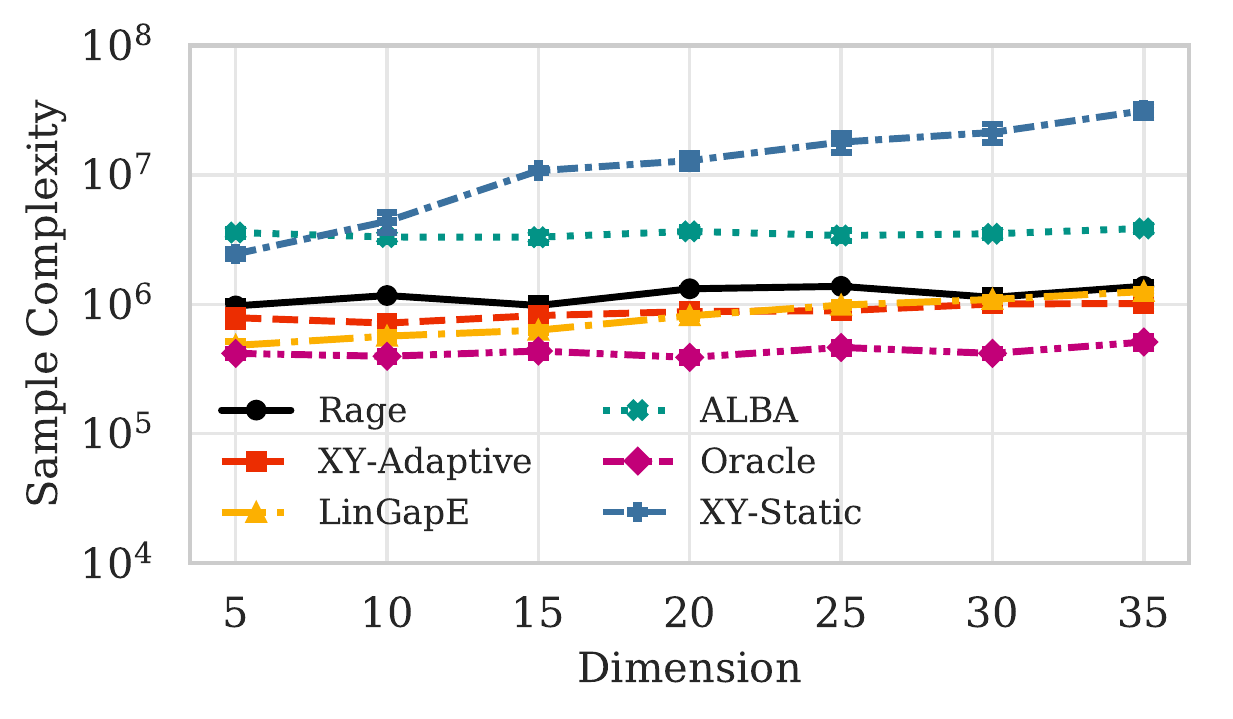}\label{fig:sincos}}
    \subfloat[][Duplicate arms]{\includegraphics[width=0.5\linewidth]{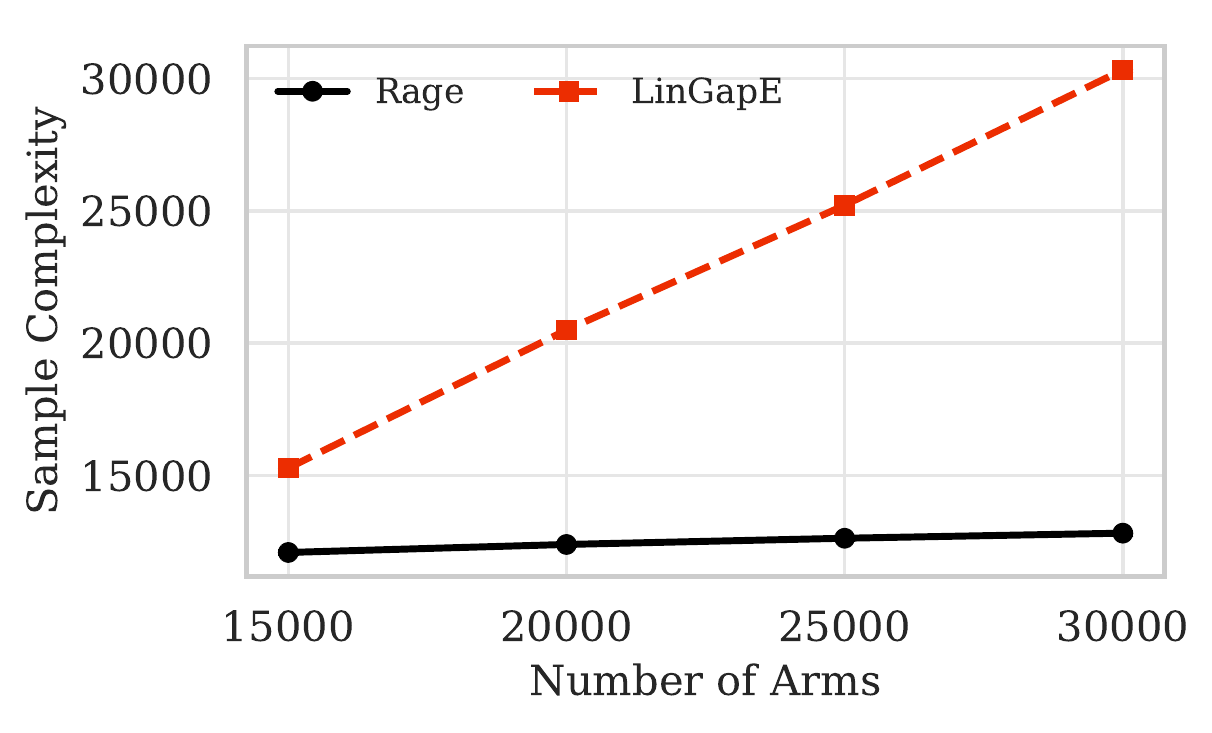}\label{fig:duplicate}}\\
    \subfloat[][Uniform sphere]{\includegraphics[width=0.5\linewidth]{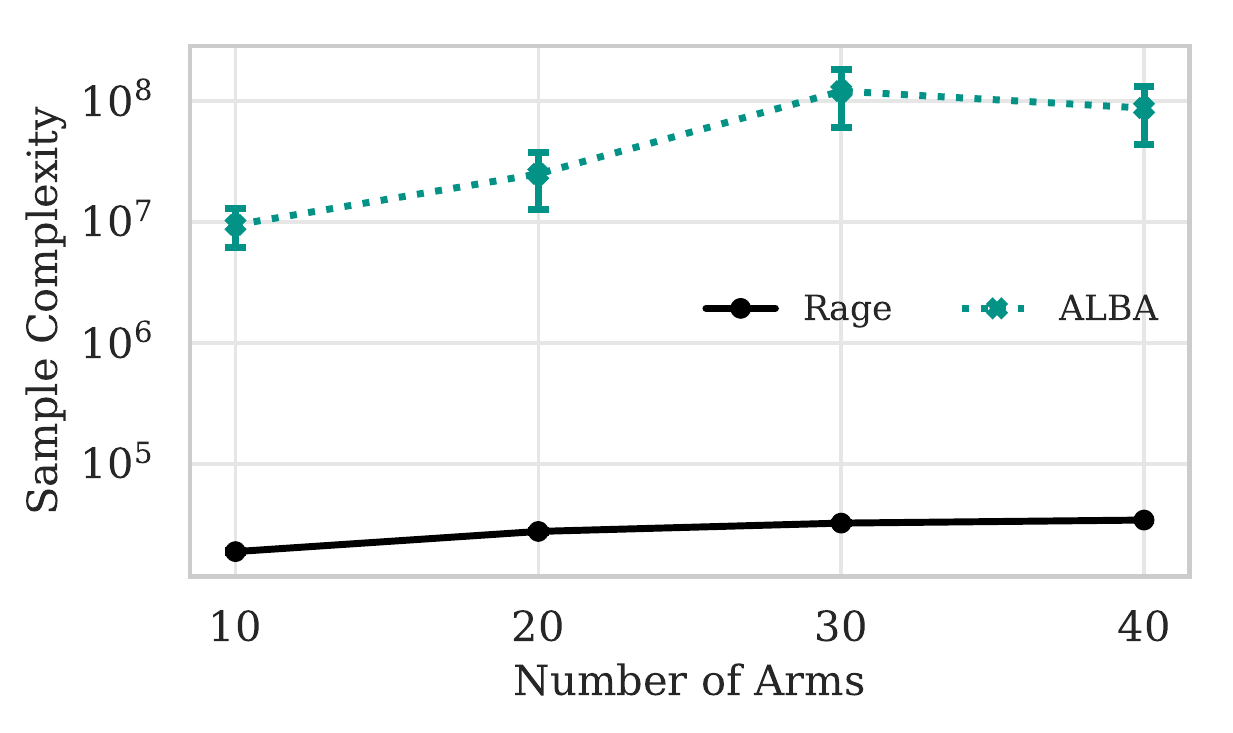}\label{fig:uniform}}
    \subfloat[][Transductive]{\includegraphics[width=0.5\linewidth]{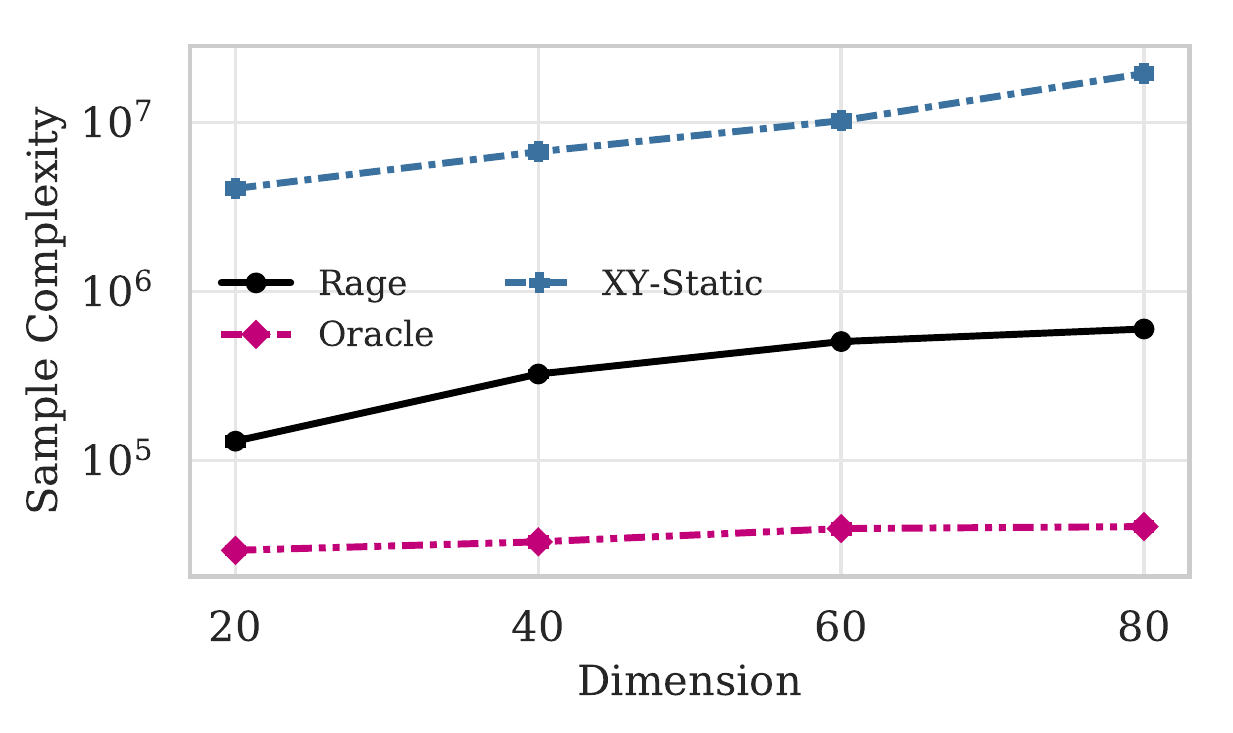}\label{fig:transductive}}
    \caption{Simulation Results}
\end{figure}
% \begin{enumerate}
% \item cos/sin example
% \item x is concentric circle on sphere around parameter
% \item Bump up n.
% \item Bunch of arms in 1 direction.
% \item Some web like thing.
% \item If time, transductive thing.
% \end{enumerate}

\section{Conclusion}
In this paper we have proposed the problem of best-arm identification for transductive linear bandits, provided an algorithm, and matching upper and lower bounds. As a remark it is straightforward to exit our algorithm early with an $\varepsilon$-good arm. It still remains to develop anytime algorithms for this problem, as has been done in pure exploration for multi-armed bandits \cite{jamieson2014lil} that do not throw out samples. In addition, we suspect our algorithm actually matches the lower-bound and the $\log(1/\Delta_{\min})$ factor is unnecessary. Finally, it is possible that some of the ideas developed here extend to the setting of regret and could be used to give instance based regret bounds for linear bandits~\cite{lattimore2016end}. We hope to explore connections to both the regret and fixed budget settings in further works.

%\section{Conclusion}
%\begin{itemize}
%\item Fully adaptive strategies that achieve the optimal sample complexity.
%\end{itemize}

\newpage
\nocite{*}
\bibliographystyle{plain}
\bibliography{neurips_2019}

\newpage
\appendix

\section{Proof of Theorem~\ref{thm:SampleComplexity}}\label{sec:SampleComplexityProof}
\begin{proof}
%We define $\sset_{\phaseindex} := \{\arm \in \xset: (\stararm-\arm)^T\starparam < 2^{-t}\}$ to be the set containing arms with gaps smaller than $2^{-(\phaseindex)}$.
Let the good event for the $\phaseindex$th round of Algorithm~\ref{alg:main} be
\[\calE_{t} := \big\{N_t\leq \max\{\lceil 8(2^{t+1})^2 \rho(\calY(\calS_t))(1+ \epsilon)\log(\tfrac{|\calZ|^2}{\delta_{t}})\rceil, r(\epsilon)\big\} \big\} \ \cap \ \{ z_{\ast} \in \calzhat_{t+1}  \} \ \cap \ \{ \calzhat_{t+1} \subseteq \calS_{t + 1}\}\]
where we recall that $\calS_t = \{z\in \calZ: \Delta(z)\leq 2^{-t}\}$.
The good event characterizes the worst-case sample complexity of the $t$-th phase of Algorithm~\ref{alg:main} and guarantees that the set of active arms at the end of the phase contains the optimal arm and is contained in the set of arms with gaps below the threshold that is to be eliminated in the phase.
Note that for $t > \log_2(1/\Delta_{\min})$ we have $S_t = \{ z_{\ast} \}$.

The proof proceeds as follows. We begin by showing that the good event holds with probability at least $1-\delta_{t}$ in phase $t$ given that the good event held in phase $t-1$. We then show that the probability of the good event holding in every phase is at least $1-\delta$. As a result, we simply sum over the bound on the sample complexity in each phase given in the good event to obtain the stated bound on the sample complexity.

The following lemma shows that good event holds in phase $t$ with probability at least $1-\delta_{t}$ conditioned on the good event holding in phase $\phaseindex-1$.

\begin{lemma}\label{lemma:condition_delta}
$\prob(\calE_{t}|\calE_{t-1},\cdots,  \calE_{1}) \geq 1- \delta_{t}$.
\end{lemma}
%The set $\yset(\xhatset_{\phaseindex}):= \{\yarm = (\arm - \arm'): \arm \neq \arm' \in \xhatset_{\phaseindex}\}$ contains the active directions obtained from the differences of each pair of active arms remaining in phase $\phaseindex$ of the algorithm. From Proposition~\ref{prop:conf_bound}, with probability at least $1-\delta_{\phaseindex}$, it holds for each $\yarm \in \yset(\xhatset_{\phaseindex})$ that
%\begin{equation}
%\yarm^T(\starparam - \hatparam_{\armseq_{\randnumsamples_{\phaseindex}}}) \leq %\sqrt{\|\yarm\|_{A_{\armseq_{\randnumsamples_{\phaseindex}}}^{-1}}^2\log(\numarms^2/\delta_{\phaseindex})}.
%\label{eq:confidence_bound}
%\end{equation}
\begin{proof}
Conditioned on a choice of $\calY(\zhatset_t)$, since $\widehat\theta$ is a
least squares estimator of $\theta^{\ast}$ and the noise is i.i.d., we know that $y^{\top}(\theta^{\ast} - \widehat\theta_t)$ is  $\|y\|_{A_{t}^{-1}}^2$-subGaussian for all $y\in \calY(\calzhat_t)$.
Furthermore, due to the guarantees of the rounding procedure, $\|y\|_{A_{t}^{-1}}^2\leq (1+\epsilon) \rho(\calY(\zhatset_t))/N_t \leq \left(8(2^{t+1})^2\log(|\calZ|^2/\delta_t)\right)^{-1}$ for all $y\in \calY(\widehat\calZ_t)$ by our choice of $N_t$. Since the right-hand side is deterministic, independent of $\calY(\zhatset_t)$,
%we have that
%In particular, by our choice of $N_t$, this implies that $y^T(\theta^{\ast} - \widehat\theta)$ is $((2^{t+1})^2\log(4K^2/\delta_t))^{-1}$-subGaussian, independent of $\calY(\xhatset_t)$.
%However, by our choice of $N_t$, $\rho(\calY(\zhatset_t))/N_t\leq ((2^{t+1})^2(1+\epsilon)\log(4K^2/\delta_t))^{-1}$, where the right hand side is deterministic.
for any $\nu > 0$, we have that
%\P\left(|y^T(\thetastar - \thetahat)| > \sqrt{2\|y\|_{A_{t}^{-1}}^2 \log(2/\nu)}\Bigg|\calE_{t-1}\right) <
\[\P\left(|y^{\top}(\thetastar - \thetahat)| > \sqrt{\frac{2\log(2/\nu)}{8(2^{t+1})^2\log(|\calZ|^2/\delta_t)}}\Bigg|\calE_{t-1},\cdots,  \calE_{1}\right)\leq \nu\]
for any $y\in \calY(\zhatset_t)$.
Taking $\nu = 2\delta_t/|\calZ|^2$ and union bounding over all the possible $y\in \calY(\zhatset_t)$ where $|\calY(\zhatset_t)| \leq |\calY(\calZ)| \leq |\calZ|^2/2$, gives us that
\begin{equation}\label{eq:ConfidenceBound}
\P(\exists y\in \calY(\widehat{\calZ}_t)\quad|y^{\top}(\thetastar - \thetahat)|
> 2^{-(t+2)}|\calE_{t-1},\cdots, \calE_{1})\leq \delta_t.
\end{equation}

{\emph{Claim 1:}} Every arm $z \in \zhatset_{t}$ such that $\Delta(z) \geq 2^{-(t+1)}$ is discarded in phase $t$ so that $\zhatset_{\phaseindex+1} \subseteq \sset_{t+1}$ with probability at least $1-\delta_{t}$.

{\textbf{Proof}.} %Indeed, from~\eqref{eq:confidence_bound} we have that for each $\yarm \in \yset(\zhatset_{t})$,
%\[\yarm^{\top}(\starparam - \hatparam_{t}) \leq 2^{-(t+2)} \iff
%\yarm^{\top}\hatparam_{t} \geq \yarm^{\top}\starparam - 2^{-(t+1)}.\]
Since we conditioned on $\event_{t-1}$ , $z_{\ast} \in \zhatset_{t}$.
If $z\in \calS_{t+1}^c\cap \zhatset_t$ then by definition $\Delta(z) = (z_{\ast} - z)^{\top}\theta^{\ast} > 2^{-(t+1)}$. Taking $y = z_{\ast}-z$ by the confidence bound ~\eqref{eq:ConfidenceBound}
\[\yarm^{\top}(\theta^{\ast} - \widehat{\theta}_{t}) \leq 2^{-(t+2)} \Rightarrow y^\top\widehat{\theta}_{t} \geq y^{\top}\thetastar - 2^{-(t+2)} > 2^{-(t+1)} - 2^{-(t+2)} = 2^{-(t+2)}.\]
However, this is precisely the discard condition of the algorithm guaranteeing $z$ will be eliminated. %An arm $\arm \in \xhatset_{\phaseindex}$ is discarded in Algorithm~\ref{alg:subroutine} if there exists a $\yarm = \arm'-\arm$ for $\arm'\neq\arm \in \xhatset_{\phaseindex}$ such that
%\begin{equation}
%2^{-(t+1)}  \leq \yarm^T\hatparam_{\armseq_{\randnumsamples_{\phaseindex}}}.
%\label{eq:discard}
%\end{equation}
%Rearranging terms, we get that for each $\yarm \in \yset(\xhatset_{\phaseindex})$,
%\begin{equation*}
%\yarm^T\hatparam_{\armseq_{\randnumsamples_{\phaseindex}}} \geq \yarm^T\starparam - 2^{-(\phaseindex+1)} \quad
%\text{and} \quad \sqrt{\|\yarm\|_{A_{\armseq_{\randnumsamples_{\phaseindex}}}^{-1}}^2\log(\numarms^2/\delta_{\phaseindex})} \leq 2^{-(\phaseindex+1)}.
%\end{equation*}
%Plugging the bounds from the previous equation into~\eqref{eq:discard}, we can conclude that the discard condition becomes
%\begin{equation*}
%2^{-(\phaseindex+1)} \leq \yarm^T\starparam - 2^{-(\phaseindex+1)} \iff \yarm^T\starparam \geq 2^{-\phaseindex}.
%\end{equation*}
%Since $\event_{\phaseindex-1}$ is given, $\stararm \in \xhatset_{\phaseindex}$ and this immediately implies that every arm $\arm \in \xhatset_{\phaseindex}$ such that $(\stararm-\arm)^T\starparam  \geq 2^{-\phaseindex}$ is discarded in phase $\phaseindex$ with probability at least $1-\delta_{\phaseindex}$. We can conclude that $\xhatset_{\phaseindex+1}\subseteq \sset_{\phaseindex+1}$ with probability at least $1-\delta_{\phaseindex}$ since $\sset_{\phaseindex+1}^c = \{\arm: (\stararm - \arm)^T\starparam \geq 2^{-\phaseindex}\}$.

We now show that the optimal arm cannot be discarded in a phase with high probability. \qed

{\emph{Claim 2:}} $z_{\ast} \in \calzhat_{t+1}$ with probability at least $1-\delta_{t}$.

{\textbf{Proof}.}  We prove this claim by contradiction. To begin, observe that $z_{\ast}$ is in $\zhatset_{t}$ since $\calE_{t-1}$ holds. Now, suppose that $z_{\ast}$ is discarded in phase $t$. This implies that there exists a $z \neq z_{\ast}$ for $z \in \zhatset_{t}$ such that $2^{-(t+2)}\leq (z-z_{\ast})^{\top}\hatparam_{t}$.
However from the confidence interval~\eqref{eq:ConfidenceBound}, $(z-z_{\ast})^{\top}(\widehat{\theta}_{t}-\theta^{\ast}) \leq 2^{-(t+2)}$. Combining these we see that $(z-z_{\ast})^{\top}(\widehat{\theta}_{t}-\theta^{\ast}) < (z-z_{\ast})^\top\widehat{\theta}_t$
which implies $(z-z_{\ast})^{\top}\theta^{\ast} > 0$ which is a contradiction.\qed

%\begin{equation*}
%\yarm^T(\hatparam_{\armseq_{\randnumsamples_{\phaseindex}}}-\starparam) \leq \sqrt{\|\yarm\|_{A_{\armseq_{\randnumsamples_{\phaseindex}}}^{-1}}^2\log(\numarms^2/\delta_{\phaseindex})}  \leq \yarm^T\hatparam_{\armseq_{\randnumsamples_{\phaseindex}}}.
%\end{equation*}
%Equivalently, we obtain $\yarm^{\top}\starparam \geq 0 \iff z^{\top}\starparam \geq (z_{\ast})^{\top}\starparam$. This condition cannot hold since $z_{\ast}$ is the optimal arm and gives a contradiction. As a result, we can conclude that $z_{\ast}$ cannot be discarded and $z_{\ast} \in \zhatset_{\phaseindex+1}$ with probability at least $1-\delta_{t}$.

We complete the proof by showing that the sample complexity of phase $t$ given in the good event holds with probability $1-\delta_{t}$. Since $\calE_{t-1}$ is given, $\widehat{\calZ}_{t} \subseteq \sset_{t}$, which implies with probability at least $1-\delta_{t}$,
\begin{align*}
N_{t} &=
\max\big\{\big\lceil 8 (2^{t+1})^2\rho(\yset(\zhatset_{t}))(1+\varepsilon)\log(|\calZ|^2/\delta_{t})\big\rceil, r(\epsilon)\big\} \\
&\leq \max\big\{\big\lceil 8 (2^{t+1})^2\rho(\yset(\calS_{t}))(1+\varepsilon)\log(|\calZ|^2/\delta_{t})\big\rceil, r(\epsilon)\big\}
\end{align*}
where we note that the quantity on the right hand side is deterministic.
\end{proof}

\begin{lemma}\label{lemma:all_clean}
$\prob(\event_{1}\cap \dots \cap \event_{\lceil\log_{2}(1/\Delta_{\min})\rceil}) \geq 1- \delta$.
\end{lemma}
\begin{proof}
Let us first expand the intersection of the events into a product of conditional
probabilities as follows:
\begin{equation*}
\prob(\event_{1}\cap \dots \cap \event_{\lceil\log_{2}(1/\Delta_{\min})\rceil}) = \Pi_{\phaseindex=1}^{\lceil\log_{2}(1/\Delta_{\min})\rceil} \prob(\event_{\phaseindex}|\event_{\phaseindex-1}\cap\dots\cap \event_{1})
\end{equation*}
We now obtain a lower bound on the success probability using Lemma~\ref{lemma:condition_delta} and facts about infinite products:
\begin{equation*}
\Pi_{\phaseindex=1}^{\lceil\log_{2}(1/\Delta_{\min})\rceil} \prob(\event_{\phaseindex}|\event_{\phaseindex-1}\cap\dots\cap \event_{1}) \geq \Pi_{\phaseindex=1}^{\lceil\log_{2}(1/\Delta_{\min})\rceil} (1-\delta_{\phaseindex})
\geq \Pi_{t=1}^{\infty} \Big(1-\frac{\delta}{t^2}\Big)
= \frac{\sin(\pi\delta)}{\pi\delta}.
\end{equation*}
Finally, using the fact that $\frac{\sin(\pi\delta)}{\pi\delta} \geq 1- \delta$ for $\delta\in (0, 1)$, we obtain the result
$\prob(\event_{1}\cap \dots \cap \event_{\lceil\log_{2}(1/\Delta_{\min})\rceil}) \geq 1- \delta$.
\end{proof}

The final result then follows immediately from Lemmas~\ref{lemma:condition_delta} and~\ref{lemma:all_clean} since we can now sum the number of samples taken in each phase to get the sample complexity. With probability at least $1-\delta$,
\begin{align*}
\randnumsamples &\leq \sum_{\phaseindex=1}^{\lfloor \log_2(1/\Delta_{\min}) \rfloor} \max\big\{\big\lceil8 (2^{\phaseindex+1})^2 \rho(\yset(\sset_{\phaseindex}))(1+ \epsilon)\log(t^2|\calZ|^2/\delta)\big\rceil, r(\epsilon)\big\} \\
&\leq 128 \psi^\ast (1+ \epsilon) \log(1/\Delta_{\min}) \log(\log_2(1/\Delta_{\min})^2|\calZ|^2/\delta)  +  (1+r(\epsilon)) \log_2(1/\Delta_{\min}).
\end{align*}

Recall that $\calY^{\ast}(\calS) = \{z_{\ast} -z: \forall \ z \in \calS \setminus z_*\}$. To see the second inequality, note that
\begin{align*}
\psi^\ast &= \min_{\lambda \in \simplex_{\calX}} \max_{y \in \calY^{\ast}(\calZ)} \frac{\|y\|_{(\sum_{x\in \calX} \lambda_x xx^{\top})^{-1}}^2}{\Delta(y)^2} \\
&= \min_{\lambda \in \simplex_{\calX}} \max_{t\leq \lfloor\log_2(1/\Delta_{\min})\rfloor}\max_{y \in \calY^{\ast}(\calS_t)}\frac{\|y\|_{(\sum_{x\in \calX} \lambda_x xx^{\top})^{-1}}^2}{\Delta(y)^2} \\
&\geq \min_{\lambda \in \simplex_{\calX}}\max_{t\leq \lfloor\log(1/\Delta_{\min})\rfloor}\max_{y \in \calY^{\ast}(\calS_t)}\frac{\|y\|_{(\sum_{x\in \calX} \lambda_x xx^{\top})^{-1}}^2}{(2^{-t})^2} \\
&\overset{(i)}{\geq}  \frac{1}{\log_2(1/\Delta_{\min})}  \min_{\lambda\in\simplex_{\calX}} \sum_{t=1}^{\lfloor\log_2(1/\Delta_{\min})\rfloor} \max_{y \in \calY^{\ast}(\calS_t)} \frac{\|\yarm\|_{(\sum_{x\in \calX} \lambda_x xx^{\top})^{-1}}^2}{(2^{-t})^2}\\
&\overset{(ii)}{\geq} \frac{1}{\log_2(1/\Delta_{\min})} \sum_{t=1}^{\lfloor\log_2(1/\Delta_{\min})\rfloor}  (2^{t})^2 \min_{\lambda\in \simplex_{\calX}} \max_{\yarm \in \yset^{\ast}(\sset_t)} \|\yarm\|_{(\sum_{x\in \calX} \lambda_x xx^{\top})^{-1}}^2  \\
&\overset{(iii)}{\geq} \frac{1}{4 \log_2(1/\Delta_{\min})} \sum_{t=1}^{\lfloor\log_2(1/\Delta_{\min})\rfloor} (2^{t})^2 \min_{\lambda\in \simplex_{\calX}} \max_{\yarm \in \yset(\sset_t)} \|\yarm\|_{(\sum_{x\in \calX} \lambda_x xx^{\top})^{-1}}^2  \\
&= \frac{1}{4 \log_2(1/\Delta_{\min})}\sum_{i=1}^{\lfloor\log(1/\Delta_{\min})\rfloor} (2^{t})^2\rho(\calY(S_t))
\end{align*}
where $(i)$ follows from the fact that the maximum of positive numbers is always less than the average, and $(ii)$ by the fact that the minimum of a sum is greater than the sum of minimums.
To see $(iii)$, note that for $y\in \calY(\calS_t)$, if $y=z_i-z_j$, then $y = (z_i - z_{\ast}) - (z_{\ast} - z_j)$.
Hence $\max_{y\in \calY(\calS_t)} \|y\|_{(\sum_{x\in \calX} \lambda_x xx^{\top})^{-1}}^2 \leq 4 \max_{y\in \calY^{\ast}(\calS_t)} \|y\|_{(\sum_{x\in \calX} \lambda_x xx^{\top})^{-1}}^2$.
\end{proof}

\section{Efficient Rounding Procedures}\label{sec:rounding}

%Similar to the discussion in Section~\ref{sec:LinearExperimentalDesign}, given a fixed design of $T$ samples $\mathbf{x}_T = \{x_{I_1}, x_{I_2}, \cdots, x_{I_T}\}, x_{I_t}\in \calX$, and a collection of vectors $\calY$ we can construct a confidence interval on each $y^\top\theta^\ast$ using the ordinary least squares estimator $\widehat\theta$,
%\begin{align}
%    \P\Big(\exists y\in \calY: |y^\top(\theta^{\ast} - \widehat\theta)| > \sqrt{y^\top(\sum_{i=1}^n \lambda_i x_{I_t} x_{I_t}^\top)^{-1}y\log(4|\calY|/\delta)}\Big) > 1-\delta
%\end{align}
%For any finite budget of $T$ samples, we are potentially interested in using the design
%\[\mathbf{x}_T =\min_{\mathbf{x}_T}\max_{y\in \calY} y^\top\big(\sum_{i=1}^T  x_{I_t} x_{I_t}^\top\big)^{-1}y\]
%and choosing as a minimizer of the above.
Throughout the following we assume that $\calY\subset \mathbb{R}^d$ is arbitrary and that $\calX=\{x_1, \cdots, x_n\}\subset \mathbb{R}^d$ is a subset with $\dim\text{span}(\calX) = d$.

\begin{definition}
A \emph{rounding procedure} is an algorithm that takes as input $\lambda\in
\simplex^n$, a set of vectors $\calX$, and a number of samples $N$ and returns a
finite \textbf{allocation} $s=(s_1, \cdots, s_n)\in \mathbb{N}^n$ satisfying the
following properties: 1. $\sum_{i=1}^n s_i = N$; 2. there exists a function
$r(\epsilon)$ such that if $N > r(\epsilon)$, then
    $\max_{y\in \calY} \|y\|^2_{(\sum_{i=1}^n s_i x_i x_i^{\top})^{-1}} \leq (1+\epsilon) \max_{y\in \calY} \|y\|^2_{(\sum_{i=1}^n \lambda_i x_i x_i^{\top})^{-1}}/N$.
\end{definition}

%The problem of constructing an allocation on $T$ samples, $x_T$
%To reduce our uncertainty in each direction $y\in \calY$, this motivates defining
%\[\rho_T(\calY) =\min_{\mathbf{x}_T}\max_{y\in \calY} y^\top\big(\sum_{i=1}^T  x_{I_t} x_{I_t}^\top\big)^{-1}y\]
%and choosing $\mathbf{x}_T$ as a minimizer of the above.
Fortunately, there has been extensive work on efficient rounding procedures, motivated by the strong connection to G-optimal design in optimal linear experimental design \cite{pukelsheim2006optimal}. %In general this combinatorial problem is NP-Hard. A natural continuous relaxation of this problem is given by
%\[\rho(\calY) := \min_{\lambda\in \Sigma_n}\max_{y\in \calY} y^\top(\sum_{i=1}^n \lambda_i x_i x_i^\top)^{-1}y.\]
%We let $\lambda^{\ast}(\calY)$ to be any distribution over $\calX$ that acheives the optimal value.
%In the experimental design literature a wide number of approximate solutions have been proposed tosolve the NP-hard discrete optimization problem in Eq. 12 (see [4, 17] for some recent results and [18] for a more thorough discussion).
%We may hope that given $x_i\in \calX$, $\lceil T\lambda_i \rceil$ samples gives an optimal allocation, however of course there is no guarantee that the number of samples $\sum_{i=1}^{n} \lceil T\lambda_i \rceil$ is $T$.
%Throughout the following we will also need access to an algorithm that can give us a efficient allocation of $N$ samples given a distribution $\lambda$ over the $\calX$'s.
%\textcolor{red}{Define the discrete allocation as $\rho_{N_t}$}
Here we discuss two important rounding procedures. The first is due to \cite{pukelsheim2006optimal} and has an $r(\epsilon) = d^2/\epsilon$ where the $d^2$ arises from the support size of $\lambda$.

{\bf Rounding Procedure of \cite{pukelsheim2006optimal}.} An efficient rounding
procedure is given in Chapter 12 of~\cite{pukelsheim2006optimal} to transform a design $\design \in \simplex^n$ into a discrete allocation $s \in \mathbb{N}^n$ for any fixed number of samples $N$.
%Let $\support$ denote the cardinality of the support of $\design$.
%, where the
%support of $\design$ is defined as the set of arms $\arm_{\armindex} \in \xset$
%for which $\design_{\armindex}> 0$.
The rounding procedure determines the number of pulls $N_i$ to allocate to each
arm $\arm_{\armindex}$ in the support of $\design$ such that
$\sum_{\armindex\leq \support} N_{\armindex}=N$ where $\support$ is the
cardinality of the support of $\design$.
%We refer the reader to the cited sections for a detailed description.
The discrete allocation from the rounding procedure is obtained in two phases:
\begin{enumerate}[topsep=-2pt]
\item Given the number of samples $N$ to obtain and the cardinality of the support of $\design$, samples to allocate to arms in the support of $\design$ are computed using $N_{\armindex} = \lceil (N - \tfrac{1}{2}\support)\design_{\armindex} \rceil$, where $N_1, N_2, \dots, N_{\support}$ are positive integers constrained such that $\sum_{\armindex\leq \support} N_{\armindex}\geq N$.
\item Following the previous phase of the rounding procedure, loop until the
    discrepancy $(\sum_{\armindex\leq \support} N_{\armindex})-N=0$, from either
    decreasing a sample count $N_j$ which obtains
    $(N_{\sampindex}-1)/\design_{\sampindex}=\min_{\armindex\leq \support}(N_i -
    1)/\design_{\armindex}$ to $N_{\sampindex}+1$, or increasing a sample count
    $N_{\sampindex}$ which obtains
    $N_{\sampindex}/\design_{\sampindex}=\max_{\armindex \leq \support}
    N_i/\design_{\armindex}$ to $N_j-1$.
    %{\color{red} [I have no idea if this is
    %right, or waht hte next paragaph says]}\textcolor{blue}{[I think the
    %procedure is now right.]}
\end{enumerate}
The efficient design apportionment theorem in~Section 12.5
of~\cite{pukelsheim2006optimal} provides the foundation the procedure; details
on the characterization of efficiency are provided in Chapter 12 of the same
reference.
%its
%efficiency is characterized in the theorem in~Section 12.7 of the same
%reference.
%further details can be found in Sections 12.4 and 12.5
%of the same reference where the procedure is described.

% The rounding procedure produces a discrete allocation that is monotonic in the number of samples $N$.
% This means the discrete allocation $\armseq_{\numsamples+1}$ only deviates from the discrete allocation $\armseq_{\numsamples+1}$ from an arm being given one more pull.
% This gives rise to a simple incremental rule to update the discrete allocation.

%\begin{theorem}
%
%\end{theorem}

\textbf{Rounding Procedure of \cite{allen2017near}.} We refer the reader to Algorithm 1 in \cite{allen2017near} for details about their rounding procedure. Here we describe their result and how to modify it to our setting. Let $\mathcal{S}_{b,k} = \{s\in [b]^n: \sum_{i=1}^n s_i \leq k\}$ and a continuous relaxation $\mathcal{C}_{b,k} = \{s\in [0,b]^n: \sum_{i=1}^n s_i \leq k\}$.
\begin{theorem}[Theorem 2.1 of \cite{allen2017near}]
    Suppose $\epsilon \in (0,1/3]$, $n\geq k\geq 180 d/\epsilon^2$, $b\in [k]$. Let $\pi\in C_{b,k}$, then in polynomial-time (in $n$ and $d$) we can round $\pi$ to an integral solution $\widehat{s} \in S_{b,k}$ satisfying $\max_{y\in \calY} \|y\|^2_{(\sum \widehat{s}_i x_i x_i^{\top})^{-1}} \leq (1+\epsilon) \max_{y\in \calY} \|y\|^2_{(\sum \pi_i x_i x_i^{\top})^{-1}}$.
\end{theorem}

To apply this theorem to obtain an efficient rounding procedure, consider the following. Given a $\lambda\in \simplex_{\calX}$, and a number of samples $N$, let $\pi = N\lambda$ and consider the case where $b=k=N$. Then $k\lambda\in C_{k,k}$. In general the theorem does not allow $N=k>n$, but we can circumvent this by just duplicating each vector in $\calX$ exactly $N$ times. Then the allocation $\widehat{s}$ obtained will satisfy the conditions of the above with $r(\epsilon) = 180d/\epsilon^2$. The authors remark that it is most likely true that $r(\epsilon) = d/\epsilon^2$ suffices, but we are not aware of any such result in the literature.

\section{Proof of Theorem \ref{thm:LowerBound}}\label{sec:proofOfLowerBound}
\begin{proof}
In this section we assume $\calX = \{x_1, \cdots, x_n\}$ and $\calZ = \{z_1, \cdots, z_m\}$.
Without loss of generality, we assume that $z_1 = \argmax_{z_i\in \calZ} z_i^{\top}\theta^{\ast}$.
Let $\calC := \{\theta\in \mathbb{R}^d: \exists i \text{ s.t. }
\theta^{\top}(z_1 - z_i) \leq 0\}$, i.e. $\theta\in \calC$ if and only if $z_1$ is not the best arm in the linear bandit instance $(\calX, \calZ, \theta)$.

We now recall the transportation lemma of \cite{kaufmann2016complexity}.
Under a $\delta$-PAC strategy for finding the best arm for the bandit instance $(\calX, \calZ, \theta^{\ast})$, let $T_i$ denote the random variable which is the number of times arm $i$ is pulled.
In addition let $\nu_{\theta, i}$ denote the reward distribution of the $i$-th arm of $\calZ$, i.e. $\nu_{\theta, i} = \mc{N}(z_i^{\top}\theta, 1)$.
Then for any $\theta\in \calC$ we have that
\begin{equation*}
    \sum_{i=1}^n \E[T_i] KL(\nu_{\theta^{\ast}, i}, \nu_{\theta, i}) \geq \log(1/2.4\delta).
\end{equation*}
In particular, $\sum_{i=1}^n \E[T_i] \geq \sum_{i=1}^n t_i$ for any $\mathbf{t}
:= (t_1, \cdots, t_n)$ which is a feasible solution of the optimization problem,
\begin{align*}
    &\min \sum_{i=1}^n t_i\quad
    \text{subject to } \min_{\theta\in \calC}\sum_{i=1}^n t_i KL(\nu_{\theta^{\ast},i}|| \nu_{\theta,i}) \geq \log(1/2.4\delta).
\end{align*}
Taking $\mathbf{t}^{\ast}$ to be an optimal solution to the previous problem, note that
\begin{align*}
    \min_{\theta \in \calC} \sum_{i=1}^n \frac{t^{\ast}_i}{\sum_{j=1}^n t^{\ast}_j} KL(\nu_{\theta^{\ast},i}|| \nu_{\theta,i})
    \geq \frac{\log(1/2.4\delta)}{\sum_{j=1}^n t^{\ast}_j}
    \geq \frac{\log(1/2.4\delta)}{\sum_{j=1}^n \E[T_j]}
\end{align*}
In particular, since $\sum_{i=1}^n \frac{t^{\ast}_i}{\sum_{j=1}^n t^{\ast}_j} = 1$, we see that
\[\max_{\lambda\in \simplex_n}\min_{\theta \in \calC} \sum_{i=1}^n \lambda_i KL(\nu_{\theta^{\ast},i}|| \nu_{\theta,i})\geq \frac{\log(1/2.4\delta)}{\sum_{i=1}^n \E[T_i]}.\]
Rearranging, we see that
\begin{align}\label{eqn:lowerboundT}
    \sum_{i=1}^n \E[T_i] \geq \log(1/2.4\delta) \min_{\lambda\in \simplex_n}\max_{\theta \in \calC} \frac{1}{\sum_{i=1}^n \lambda_i KL(\nu_{\theta^{\ast}, i }|| \nu_{\theta, i})}.
\end{align}
 Now for $j\neq 1$, $\lambda \in \simplex^{n}$ and $\epsilon > 0$,  define
 \begin{equation*}
     \theta_j(\epsilon, \lambda) = \theta^{\ast} - \frac{(y_j^{\top}\theta^{\ast}+\epsilon) A(\lambda)^{-1}y_j}{y_j^{\top} A(\lambda)^{-1} y_j}.
 \end{equation*}
where $A(\lambda) := \sum_{i=1}^n \lambda_i x_i x_i^{\top}$ and $y_j = z_1 -
z_j$. Note that $y_{j}^{\top} \theta_j(\epsilon, \lambda) = -\epsilon < 0$ which
implies that $\theta_j\in \calC$. Also, the KL-divergence is given by
\begin{align*}
    KL(\nu_{\theta^{\ast}, i }|| \nu_{\theta_j(\epsilon, \lambda), i}) &= (x_i^{\top}(\theta^{\ast} - \theta_j(\epsilon, \lambda)))^2 \\
    &= y_j^{\top} A(\lambda)^{-1} \frac{(y_j^{\top}\theta^{\ast}+\epsilon)^2 x_i
    x_i^{\top}}{(y_j^{\top} A(\lambda)^{-1} y_j)^2} A(\lambda)^{-1} y_j.
\end{align*}
Hence, returning to \eqref{eqn:lowerboundT}, we have that
\begin{align*}
    \sum_{i=1}^n \E[T_i] &\geq \log(1/2.4\delta) \min_{\lambda\in \simplex_n}\max_{\theta \in \calC} \frac{1}{\sum_{i=1}^n \lambda_i KL(\nu_{\theta^{\ast}}|| \nu_{\theta})} \\
    &\geq \log(1/2.4\delta)\min_{\lambda\in \simplex_n}\max_{j=2,\cdots,m} \frac{1}{\sum_{i=1}^n \lambda_i KL(\nu_{\theta^{\ast}, i }|| \nu_{\theta_j(\epsilon, \lambda), i})}\\
    &\geq  \log(1/2.4\delta)\min_{\lambda\in \simplex_n}\max_{j=2,\cdots,m} \frac{(y_j^{\top} A(\lambda)^{-1} y_j)^2}{(y_j^{\top}\theta^{\ast}+\epsilon)^2 y_j^{\top} A(\lambda)^{-1}(\sum_{i=1}^n \lambda_i x_i x_i^{\top})  A(\lambda)^{-1} y_j}\\
    &= \log(1/2.4\delta) \min_{\lambda\in \simplex_n}\max_{y\in \calY^{\ast}(\calZ)} \frac{y_j^{\top} A(\lambda)^{-1} y_j}{(y_j^{\top}\theta^{\ast}+\epsilon)^2 }
\end{align*}
where in the second to last line we used the fact that $\sum_{i=1}^n \lambda_i x_i x_i^{\top} = A(\lambda)$. Letting $\epsilon\rightarrow 0$ establishes the result.

{\noindent \bf Remark:}    Note that  $\theta_j = \argmin_{\theta\in \mathbb{R}^d} \|\theta - \theta^{\ast}\|^2_{A(\lambda)}  \text{ subject to }y_j^{\top}\theta = -\epsilon$.
\end{proof}

\section{Proof of Proposition~\ref{prop:fixedLowerBound}}\label{sec:fixedLowerBound}
\begin{proof}
    Assume $d$ is even and each $\epsilon_t \sim \mc{N}(0,1)$. Fix some $\alpha \in (0,1)$ which will depend on $\gamma$ in a clear way momentarily, and consider an instance where
% $\calX = \{e_i\}_{i=1}^{d}$ and $e_i$ is the $i$-th standard basis vector. Let
$\calX = \calZ = \{e_i\}_{i=1}^{d/2}\cup \{\cos(\alpha)e_i+\sin(\alpha)e_{d/2+i}\}_{i=1}^{d/2}\}$ where $e_i$ is the $i$-th standard basis vector.

% Fix a $j\leq d/2$, consider the hypothesis test between $\theta = z_j$  and $\theta = \theta_{j+d/2}$. To discriminate between these two, by standard arguments, any $\delta$-PAC algorithm that takes $N_i$ samples from the $i$-th arm $x_i = z_i$ must satisfy
% \[\max\{\P(\theta \neq \theta_j),\P(\theta \neq \theta_{j+d})\} \geq \frac{1}{2} \]
If an algorithm is $\delta$-PAC, and takes $N_i$ samples from arm $i$, then for any $j\leq d/2$ it will be able to distinguish between $\theta = z_j$  and $\theta = z_{j+d/2}$. By standard Le Cam arguments \cite{Tsybakov2004} this hypothesis test requires $N_j+N_{j+d/2} \geq \frac{c\log(1/\delta)}{(1-\cos(\alpha))^2}$ for some universal constant $c > 0$.
Because $(1-\cos(\alpha))^2 \approx \alpha^{4}/4$ and these inequalities must hold for all $j=1,\dots, d/2$ simultaneously for the single static allocation, we obtain the result.
\end{proof}

\section{Proof of Lemma~\ref{lem:interpLemma}}\label{sec:interpLemma}
\begin{proof}
    \begin{align*}
    \rho(\calY) &= \min_{\lambda\in \simplex_{|\calX|}}\max_{y\in \calY}
    \|y\|_{(\sum_{x\in\calX} \lambda_x x x^\top)^{-1}}^2\\
       &= \frac{1}{\gamma_{\calY}^2}\min_{\lambda\in
           \simplex_{|\calX|}}\max_{y\in \calY} \|y \gamma_{\calY}\|^2_{(\sum_{x\in\calX} \lambda_x x x^\top)^{-1}}\\
       &\leq \frac{1}{\gamma_{\calY}^2}\min_{\lambda\in
           \simplex_{|\calX|}}\max_{x\in \conv(\calX\cup -\calX)} \|x\|^2_{(\sum_{x\in\calX} \lambda_x x x^\top)^{-1}}\\
       &= \frac{1}{\gamma_{\calY}^2}\min_{\lambda\in
           \simplex_{|\calX|}}\max_{x\in \calX} \|x\|^2_{(\sum_{x\in\calX} \lambda_x x x^\top)^{-1}}
    \end{align*}
    The third equality follows from the fact that the maximum value of a convex function on a convex set must occur at a vertex.
    By the celebrated Kiefer-Wolfowitz theorem for G-optimal design \cite{pukelsheim2006optimal}, $\min_{\lambda\in \simplex_{|\calX|}}\max_{x\in \calX} \|x\|^2_{(\sum \lambda_x xx^{\top})^{-1}} = d$ so we see that $\rho(\calY) \leq d / \gamma_{\calY}^2$.
    For a lower bound, note that
    \begin{align*}
    \min_{\lambda\in \simplex_{\calX}}\max_{y\in \calY}  \|y\|^2_{(\sum_{x\in\calX} \lambda_x x x^\top)^{-1}}
    &\geq \min_{\lambda\in \simplex_{\calX}}\max_{y\in \calY} \sigma_{\min}((\sum_{x\in\calX} \lambda_x x x^\top)^{-1})\|y\|_2^2\\
    &= \min_{\lambda\in \simplex_{\calX}}\max_{y\in \calY} \|y\|_2^2/\sigma_{\max}((\sum_{x\in\calX} \lambda_x x x^\top)^{-1})
    \end{align*}
    where $\sigma_{\max}$ and $\sigma_{\min}$ are respectively the largest and
    smallest eigenvalue operators of a matrix. Since $\sigma_{\max}(\sum_{i=1}^n \lambda_i x_i x_i^{\top}) \leq \max_{x\in\calX} \|x\|_2$, we have that
    $\rho(\calY(S_t)) \geq \max_{y\in \calY(S_t)} \|y\|_2^2/( \max_{x\in\calX} \|x\|_2)$. The final statement in the case of a singleton is also known as Elfving's Theorem, see Section 2.14 in \cite{pukelsheim2006optimal}
\end{proof}

\section{Experiment Details}\label{sec:implementation_details}
In this section, we provide further details on the implementation of each algorithm. Each experiment was repeated 20 times with the mean sample complexity is reported and error bars representing the standard error are plotted. Noise in the observations was generated from a standard normal distribution as described in the text. Simulations were implemented in Python 3 and parallelized on an Intel(R) Xeon(R) CPU E5-2690.

For each algorithm that requires computing a design $\lambda$ from an optimization of the form $\min_{\lambda\in \simplex_{\calX}}\max_{s\in \calS} \|s\|^2_{(\sum_{x} \lambda_x xx^\top)^{-1}}$ for $\calS\subset \mathbb{R}^d$, i.e. RAGE, XY-static, XY-oracle, and ALBA, we used a Franke-Wolfe algorithm \cite{jaggi2013revisiting} with constant step-size $2/(k+2)$ ($k$ being the iteration counter).
The algorithm was run until the relative change in $\lambda$ with respect to the $\ell_2$ norm was less than $.01$ or 1000 iterations were reached. Any values of $\lambda < 10^{-5}$ were then thresholded to 0 and $\lambda$ was scaled to sum to 1.

\begin{itemize}[topsep=0pt, partopsep=0pt, labelindent=0pt, labelwidth=0pt, leftmargin=10pt]
\item $\calX\calY$-Adaptive~\cite{soare2014best}: This algorithm requires a parameter $\alpha$ that governs the length of each adaptive phase.
We follow the simulations in~\cite{soare2014best} and let $\alpha=0.1$.
We remark that the algorithm given in the paper implements a greedy update to select arms in contrast to rounding the optimal allocation as is considered in the analysis.
We implement the greedy arm selection procedure to match the simulations in the paper.
It is worth noting that in several of the recent linear bandit papers that have implemented this algorithm, the active arm set has been reset at the conclusion of a phase before discarding arms.
We do not reset the arm set at the conclusion of a phase to match what was done in~\cite{soare2014best}.
Finally, in the confidence interval, we include the phase index and not the number of samples since we only need to union bound over when it is evaluated.
\item $\calX\calY$-Static and $\calX\calY$-Oracle: To implement each allocation, we compute the optimal design on the set $\calY(\calZ)$ for the static strategy and the set $\calY^{\ast}(\calZ)$ normalized by the gaps for the oracle.
Each algorithm proceeds in phases drawing $v^{t}$ samples from the allocation and pulling the selected arms, where $v$ was optimized in the range $(1,2)$ for performance, and the stopping condition as discussed in Section~\ref{sec:leastSquaresReview} is evaluated at the end of each phase $t$.
\item LinGapE~\cite{xu2018fully}:
We run this algorithm with a regularizer on the least squares estimator of $\lambda=1$ following the implementation given in the paper.
LinGapE is designed to find an $\varepsilon$ good arm.
We let $\varepsilon=0$ to ensure the optimal arm is identified.
We note that in the code the authors of~\cite{xu2018fully} provided with their paper, $\varepsilon$ was set to the minimum gap.
Since this value is unknown a priori, we do not follow this approach and as a result our simulations may not match with the simulations in~\cite{xu2018fully} for identical problem instances.
Moreover, the simulations in the paper apply a greedy arm selection strategy that deviates from the algorithm that is analyzed.
We instead implement the LinGapE algorithm in the form that it is analyzed.
\item ALBA~\cite{tao2018best}: This algorithm is parameter free and we implement the $\calY$-ElimTil sub-procedure following the paper since it gives improved empirical results compared to the $\calX$-ElimTil sub-procedure that provides identical theoretical results.
\item RAGE: To compute the discrete allocation given a design, we use the rounding procedure discussed in Section~\ref{sec:rounding} from~\cite{pukelsheim2006optimal}. For the Benchmark example, and uniform points on a sphere we used $\epsilon = 1/5$, for the many arms example, we used $\epsilon = 1/50$ and for the case of uniform points on a sphere, we used $\epsilon = 1/3.5$.
The algorithm we propose is computationally efficient since there is at most $\lfloor\log_2(1/\Delta_{\min})\rfloor$ phases and each phase only requires solving a convex optimization to obtain the design, an efficient rounding procedure, and solving a least squares problem.
The time required between each pull is negligible.
\end{itemize}

\end{document}